\documentclass{article}
\usepackage{ijcai25}

\usepackage{times}
\usepackage{soul}
\usepackage{url}
\usepackage{bm}
\usepackage{amssymb}
\usepackage[hidelinks]{hyperref}
\usepackage[utf8]{inputenc}
\usepackage[small]{caption}
\usepackage{graphicx}
\usepackage[caption=false]{subfig}
\usepackage{svg}
\usepackage{amsmath}
\usepackage{amsthm}
\usepackage[countmax]{subfloat}
\usepackage{booktabs}
\usepackage{algorithm}
\usepackage{algorithmic}
\usepackage[switch]{lineno}

\newtheorem{theorem}{Theorem}
\newtheorem{lemma}{Lemma}
\newtheorem{remark}{Remark}

\title{Online-BLS: An Accurate and Efficient Online Broad Learning System for Data Stream Classification}

\author{
Chunyu Lei$^1$
\and
Guang-Ze Chen$^2$\and
C. L. Philip Chen$^1$\And
Tong Zhang$^{1,}$\footnote{Corresponding Author} \\
\affiliations
$^1$South China University of Technology\\
$^2$University of Macau\\
\emails
202210188289@mail.scut.edu.cn, chenguangze1999@gmail.com,
philip.chen@ieee.org,
tony@scut.edu.cn
}

\begin{document}
\maketitle
\begin{abstract}
The state-of-the-art online learning models generally conduct a single online gradient descent when a new sample arrives and thus suffer from suboptimal model weights. To this end, we introduce an online broad learning system framework with closed-form solutions for each online update. Different from employing existing incremental broad learning algorithms for online learning tasks, which tend to incur degraded accuracy and expensive online update overhead, we design an effective weight estimation algorithm and an efficient online updating strategy to remedy the above two deficiencies, respectively. Specifically, an effective weight estimation algorithm is first developed by replacing notorious matrix inverse operations with Cholesky decomposition and forward-backward substitution to improve model accuracy. Second, we devise an efficient online updating strategy that dramatically reduces online update time. Theoretical analysis exhibits the splendid error bound and low time complexity of our model. The most popular test-then-training evaluation experiments on various real-world datasets prove its superiority and efficiency. Furthermore, our framework is naturally extended to data stream scenarios with concept drift and exceeds state-of-the-art baselines.
\end{abstract}

\section{Introduction}
Online Learning (OL)~\cite{hoi2021online} provides a crucial solution for learning knowledge from various real-world data streams, such as weather~\cite{5975223}, call record~\cite{swithinbank1977drift}, and stock market trading data~\cite{hu2015itrade}. Typical OL algorithms involve learning model weights from one by one instance~\cite{gunasekara2023survey}. The generic OL pipeline is as follows: 1) The initial model comes with zero or random weights because no samples are received. 2) When a new sample arrives, its model weights are rapidly updated to provide a more accurate prediction for subsequent arrivals~\cite{yuan2022recent}. Thus, a brilliant OL model must be efficient and accurate. While appealing, existing state-of-the-art OL models are generally optimized based on gradient descent algorithms~\cite{wen2024adaptive,su2024elastic}. To be specific, when a new sample arrives, all of them execute an error backpropagation and weight update process. Thus, none of them can achieve an optimal solution for model weights.

Recently, Broad Learning System (BLS) has been proposed as an efficient batch or incremental machine learning model~\cite{7987745,8457525,9380770}. First, BLS utilizes sparse coding to map raw data to a low-dimensional feature space to form feature nodes. The enhancement nodes are then obtained by random orthogonal transformation and non-linear mapping of the feature nodes. Finally, feature and enhancement nodes are cascaded, and then the relation weights between them and ground truth are estimated by the ridge regression algorithm. Thus, its closed-form solutions and incremental learning capabilities shed light on tackling the above OL issue. 

To address the suboptimal model weight issue, it is urgent to develop a novel OL framework based on BLS with closed-form solutions. Fortunately, various data incremental BLS algorithms~\cite{7987745,fu2022task,10086560,10533441} have been proposed to enhance accuracy and robustness. By setting the number of training samples for each increment to 1, most of them can inherently tackle online machine learning tasks. Thanks to the outstanding generalization of BLS, most of them can also rigorously guarantee the optimality for each online model update. Despite the straightforwardness, they still face challenges in online machine learning tasks. The common and most urgent issue is that all of them are based on the matrix inverse operation, still with high computational complexity and low numerical stability. It means that many errors are incurred when solving for online model weights. In other words, their accuracy and efficiency need to be improved.

In this paper, we develop a native Online-BLS framework to address the suboptimality issue of OL model weights.
To our knowledge, this paper clarifies the above issue for the first time and succeeds in solving it using our Online-BLS framework. Its technical innovations are summarized as follows. First, we deduce an Effective Weight Estimation Algorithm (EWEA) based on Cholesky factorization and forward-backward substitution, which significantly improves the accuracy of Online-BLS and offers excellent theoretical error bounds.
Second, an Efficient Online Updating Strategy (EOUS) is designed based on the rank-one update algorithm of the Cholesky factor~\cite{gill1974methods,seeger2007low}
to improve the efficiency of Online-BLS. The time complexity analysis shows a significant reduction in time overhead.
Furthermore, to demonstrate the superiority and flexibility of our Online-BLS framework, we propose a straightforward extension to achieve non-stationary data stream classification.
Finally, the superiority and efficiency of Online-BLS are empirically confirmed by a comprehensive comparison with baselines.

\section{Related Work}
\subsection{Online Learning}
OL is a family of machine learning algorithms for learning models incrementally from sequential data~\cite{hoi2021online}. Existing OL algorithms fall into two main categories. The first category refers to traditional OL algorithms modified from traditional machine learning models. Hoeffding Tree (HT) has been proposed to construct a decision tree online from data streams~\cite{10.1145/502512.502529}. Subsequently, an Adaptive Random Forest (ARF) has been developed to handle concept drift scenarios~\cite{gomes2017adaptive}. Specifically, once concept drift is detected, ARF resets its base tree to adapt to the new concept. Moreover, traditional online machine learning methods also include Passive Aggressive (PA)~\cite{crammer2006online}, Perceptron~\cite{rosenblatt1958perceptron}, Leveraging Bagging (LB)~\cite{bifet2010leveraging}, etc. Despite simplicity, all of them fail to extract effective features from streaming data and thus degrade their accuracy.

Inspired by deep learning~\cite{larochelle2009exploring,9360872}, the second one is online deep learning, which aims to learn a deep neural network online. To be specific, Online Gradient Descent (OGD) has been proposed to update its weights with gradient when a sample arrives~\cite{zinkevich2003online}. Given that network depth is difficult to determine in advance, \cite{sahoo2018online} proposed the Hedge Backpropagation (HBP) to dynamically integrate multiple layers of neural networks, which is of great significance for subsequent research. Subsequently, \cite{ashfahani2019autonomous} proposed Autonomous Deep Learning (ADL) that can construct its network structure from scratch by adjusting its depth and width without an initial one. To solve the online hyperparameter optimization challenge of neural networks, Continuously Adaptive Neural network for Data streams (CAND)~\cite{gunasekara2022online} has been proposed, which chooses the best model from a candidate pool of neural networks trained with different hyperparameters combination. Adaptive Tree-like Neural Network (ATNN) has been proposed to solve the catastrophic forgetting problem by choosing suitable positions on the backbone to grow branches for the new concept~\cite{wen2024adaptive}.
To address concept drift and sub-network optimization conflict issues, EODL has been proposed, including depth adaption and parameter adaption strategies~\cite{su2024elastic}. Despite significant progress, none of them avoids gradient descent. They execute an error backpropagation and weight update when a sample arrives. However, these algorithms are far from enough to ensure the optimality of their model weights in the subsequent sample prediction. Hence, developing an OL framework based on closed-form solutions becomes crucial.

\subsection{Incremental Broad Learning System}
Due to its efficiency and versatility, BLS has been widely used for a variety of tasks~\cite{10536023,10506974,10565294,10609504,10130786}. In particular, substantial works have been presented for incremental machine learning tasks.
The first data incremental algorithm for BLS is deduced from Greville's theory~\cite{greville1960some}, which we refer to I-BLS in this paper~\cite{7987745}. However, the performance of I-BLS dramatically degrades when learning a new class or task. To this end, TiBLS and BLS-CIL have been proposed based on residual learning and graph regularization, respectively~\cite{fu2022task,10086560}. 
Furthermore, motivated by the fact that I-BLS still fails to improve accuracy when the size of the training data is not roughly equal for each addition, Zhong et al. have proposed a Robust Incremental BLS (RI-BLS)
~\cite{10533441}. Among them, TiBLS stacks a new BLS module on the current model whenever a new task arrives. It views each new training sample as a new task in an online machine learning scenario. Therefore, a new BLS model needs to be learned. Inevitably, it imposes an expensive computational overhead and does not meet the requirements for efficient OL. In contrast, the remaining three algorithms only need to fine-tune their weights based on new arrivals, which is relatively efficient. The details of these incremental BLS algorithms are presented in Appendix A. Despite their feasibility, none of them can avoid matrix inversion operations. Thus, the defects of matrix inversion come along with them.

\section{Method}
\subsection{Problem Setting}
This paper focuses on online classification tasks. The input data sequence can be represented as $\mathbf{X}=\{(\mathbf{x}_k, \mathbf{y}_k)|k=1,2,\cdots,n\}$, where $\mathbf{x}_k \in \mathbb{R}^{d}$ denotes the $k$-th sample with $d$ feature values and $\mathbf{y}_k \in \mathbb{R}^{c}$ is the one-hot target corresponding to $\mathbf{x}_k$. $c$ is the total number of categories. 
When a sample $\mathbf{x}_k$ arrives, the prediction $\mathbf{\hat{y}}_k$ is derived first. Subsequently, the online learning model makes an update using the real label $\mathbf{y}_k$. 

\subsection{Online-BLS Framework}
In this section, we derive an Online-BLS framework. It provides a closed-form solution for each online update step and therefore resolves the issue of suboptimal online model weights.

Given the first sample $\mathbf{x}_1 \in \mathbb{R}^{d}$, the $i$-th feature node group $\mathbf{z}_{i} \in \mathbb{R}^{n_1}$ and $j$-th enhancement node group $\mathbf{h}_{j} \in \mathbb{R}^{n_3}$ are defined as
\begin{align}
\label{eq:z_i}
\mathbf{z}_{i}^{\top} &= \phi(\mathbf{x}_{1}^{\top}\mathbf{W}_{f_i}+\mathbf{\beta}_{f_i}^{\top}), i=1,2,\cdots,n_{2}, \\
\label{eq:h_j}
\mathbf{h}_{j}^{\top} &= \sigma(\mathbf{x}_{1}^{\top}\mathbf{W}_{e_j}+\mathbf{\beta}_{e_j}^{\top}), j=1,2,\cdots,n_{4}.
\end{align}
Here, $\mathbf{W}_{f_i} \in \mathbb{R}^{d \times n_1}$, $\mathbf{W}_{e_j} \in \mathbf{R}^{n_1n_2 \times n_3}$, $\mathbf{\beta}_{f_i} \in \mathbb{R}^{n_1}$, and $\mathbf{\beta}_{e_j} \in \mathbb{R}^{n_3}$ are randomly generated. $\phi(\cdot)$ and $\sigma(\cdot)$ are generally a linear transformation and a nonlinear activation, respectively. Through incorporating linear and non-linear features, the broad features $\mathbf{a}_1$ can be expressed as
\begin{equation}
\label{eq:a_1}
\mathbf{a}_{1}^{\top}\triangleq[\mathbf{z}_1^{\top},\cdots,\mathbf{z}_{n_2}^{\top},\mathbf{h}_1^{\top},\cdots,\mathbf{h}_{n_4}^{\top}]
\end{equation}
Then, the prediction of $\mathbf{x}_1$ can be calculated as
$\mathbf{\hat{y}}_1^{\top} = \mathbf{a}_1^{\top}\mathbf{W}^{(0)},$
where $\mathbf{W}^{(0)}$ is a zero matrix. After $\mathbf{y}_1$ is revealed, our optimization problem is expressed as
\begin{equation}
\label{eq:op1}
\mathbf{W}^{(1)}=\mathop{\arg\min}\limits_{\mathbf{W}}|\mathbf{a}_1^{\top}\mathbf{W} - \mathbf{y}_1^{\top}|^{2} + \lambda | \mathbf{W}|^{2}, 
\end{equation}
where $|\cdot|$ denotes the $L_2$ norm. 
Since equation \ref{eq:op1} is a convex optimization problem, we derive its gradient with respect to $\mathbf{W}$ and set it to $\mathbf{0}$. Then, we have
\begin{equation}
\label{eq:gde1}
(\mathbf{a}_1\mathbf{a}_1^{\top}+\lambda\mathbf{I})\mathbf{W}=\mathbf{a}_1\mathbf{y}_1^{\top}.
\end{equation}

When the $k$-th sample $\mathbf{x}_k(k=2,3,\cdots,n)$ arrives, $\mathbf{a}_k$ is obtained by replacing $\mathbf{x}_1$ in equations \ref{eq:z_i} and \ref{eq:h_j} with $\mathbf{x}_k$. Then, its prediction is $\mathbf{\hat{y}}_k^{\top} = \mathbf{a}_k^{\top}\mathbf{W}^{(k-1)}$. 
Suppose we retrain $\mathbf{W}^{(k)}$ with the first $k$ samples, equation \ref{eq:gde1} becomes
\begin{equation}
\label{eq:gdek}
([\mathbf{a}_1,\cdots,\mathbf{a}_k]\begin{bmatrix}
\mathbf{a}_1^{\top}\\
\vdots\\
\mathbf{a}_k^{\top}
\end{bmatrix}+\lambda\mathbf{I})\mathbf{W}=[\mathbf{a}_1,\cdots,\mathbf{a}_k]\begin{bmatrix}
\mathbf{y}_1^{\top}\\
\vdots\\
\mathbf{y}_k^{\top}
\end{bmatrix}.
\end{equation}
Let $\mathbf{K}^{(1)} = \mathbf{a}_1\mathbf{a}_1^{\top} + \lambda\mathbf{I}$, we have 
\begin{align}
\mathbf{K}^{(k)} &= [\mathbf{a}_1, \cdots, \mathbf{a}_k]
\begin{bmatrix}
\mathbf{a}_1^{\top}\\
\vdots\\
\mathbf{a}_k^{\top}
\end{bmatrix}
+\lambda\mathbf{I} \notag \\
&= \mathbf{K}^{(k-1)} + \mathbf{a}_k\mathbf{a}_k^{\top},
\end{align} 
and 
\begin{align}
\label{eq:rhs}
[\mathbf{a}_1, \cdots, \mathbf{a}_k]
\begin{bmatrix}
\mathbf{y}_1^{\top} \\
\vdots \\
\mathbf{y}_k^{\top}
\end{bmatrix} &= \sum_{i=1}^{k-1} \mathbf{a}_i\mathbf{y}_i^{\top} + \mathbf{a}_k\mathbf{y}_k^{\top} \notag \\
&= \mathbf{K}^{(k-1)}\mathbf{W}^{(k-1)} + \mathbf{a}_k\mathbf{y}_k^{\top} \notag \\ 
&=\big(\mathbf{K}^{(k)} - \mathbf{a}_k\mathbf{a}_k^{\top}\big)\mathbf{W}^{(k-1)} + \mathbf{a}_k\mathbf{y}_k^{\top} \notag \\
&=\mathbf{K}^{(k)}\mathbf{W}^{(k-1)} - \mathbf{a}_k\big(\mathbf{a}_k^{\top}\mathbf{W}^{(k-1)}-\mathbf{y}_k^{\top}\big).
\end{align}
Substituting equation \ref{eq:rhs} into \ref{eq:gdek}, we have
\begin{equation}
\label{eq:gdek_convert}
\mathbf{K}^{(k)}\mathbf{\Delta W}= \mathbf{B}^{(k)},
\end{equation}
where $\mathbf{\Delta W}=\mathbf{W}-\mathbf{W}^{(k-1)}$ and $\mathbf{B}^{(k)}=\mathbf{a}_k\big(\mathbf{y}_k^{\top}-\mathbf{a}_k^{\top}\mathbf{W}^{(k-1)}\big)$. 

Importantly, the right-hand terms of equations \ref{eq:gde1} and \ref{eq:gdek_convert} are equal when $k$ is equal to $1$ (i.e., $\mathbf{B}^{(1)}=\mathbf{a}_1\mathbf{y_1^{\top}}$). Thus, we conclude that the matrix $\mathbf{W}^{(1)}$ solved via equations \ref{eq:gde1} and \ref{eq:gdek_convert} is equivalent if and only if $\mathbf{W}^{(0)}=\mathbf{0}$ and $\mathbf{K}^{(0)}=\lambda\mathbf{I}$. In other words, we only need to solve equation \ref{eq:gdek_convert} to get $\mathbf{\Delta}\mathbf{W}^{(k)}$. Then, we have
\begin{equation}
\label{eq:w^k}
\mathbf{W}^{(k)} = \mathbf{W}^{(k-1)} + \mathbf{\Delta}\mathbf{W}^{(k)}, \quad k=1,2,\cdots,n.
\end{equation}
Since equation \ref{eq:gdek_convert} has a closed-form solution, our Online-BLS framework is completed without a gradient descent step. The algorithms for accurately and efficiently deriving closed-form solutions to equation \ref{eq:gdek_convert} will be described in later sections.

\subsection{Effective Weight Estimation Algorithm}
To solve $\mathbf{\Delta}\mathbf{W}$ accurately, we propose an efficient weight estimation algorithm based on Cholesky factorization and forward-backward substitution.
After the Cholesky factorization of $\mathbf{K}^{(k)}$ to obtain the Cholesky factor $\mathbf{L}^{(k)}$, equation \ref{eq:gdek_convert} can be solved via the following two equations:
\begin{align}
\label{eq:spk1}
\mathbf{L}^{(k)}\mathbf{C}&=\mathbf{B}^{(k)}, \\
\label{eq:spk2}
\big(\mathbf{L}^{(k)}\big)^{\top}\mathbf{\Delta W}&=\mathbf{C}.
\end{align}
Denotes $\mathbf{B}^{(k)}=[\mathbf{b}_{1},\cdots,\mathbf{b}_{m}]^{\top}$, using the forward and backward
substitution, the solutions of equations \ref{eq:spk1} and \ref{eq:spk2} are
\begin{equation}
\label{eq:fsk}
\mathbf{c}_{i}^{\top} = \big(\mathbf{b}_{i}^{\top}- \sum_{d=1}^{i-1}l_{di}^{(k)}\mathbf{c}_d^{\top}\big)/l_{ii}^{(k)}, i=1,\cdots,m
\end{equation}
and
\begin{align}
\label{eq:bsk}
\big(\mathbf{\Delta w}_{i}^{(k)}\big)^{\top}=\Big(\mathbf{c}_i^{\top}- \sum_{d=i+1}^{m}l_{id}^{(k)}\big(\mathbf{\Delta}&\mathbf{ w}^{(k)}_d\big)^{\top}\Big)/l^{(k)}_{ii}, \notag \\
&i=m,\cdots,1,
\end{align}
respectively. Then, our weight matrix $\mathbf{W}^{(k)}$ can be obtained online by equation \ref{eq:w^k} without the numerically unstable matrix inverse operation. The error bounds are discussed in Section \ref{section:eb}.

\subsection{Efficient Online Update Strategy}
To avoid expensive Cholesky decompositions for each online update, we discuss below how to efficiently derive $\mathbf{L}^{(k)}$ from $\mathbf{L}^{(k-1)}$ and $\mathbf{a}_k$.
Inspired by the rank-one update strategy~\cite{seeger2007low}, we first cascade the $k$-th broad feature $\mathbf{a}_k$ and the previous Cholesky factor $\mathbf{L}^{(k-1)}$ as
$[\mathbf{a}_{k},
\mathbf{L}^{(k-1)}]$.
Undergoing a series of orthogonal Givens rotations, we have
\begin{equation}
\label{eq:gives}
\mathbf{G}
\left[\begin{array}{c}
\mathbf{a}_k^{\top} \\
\big(\mathbf{L}^{(k-1)}\big)^{\top}
\end{array}\right] = \left[\begin{array}{c}
\mathbf{0}^{\top} \\
\mathbf{\bar{L}^{\top}}
\end{array}\right],
\end{equation}
where $\mathbf{G}$ is a sequence of Givens matrices of the form $\mathbf{G} = \mathbf{G}_m\mathbf{G}_{m-1}\cdots\mathbf{G}_1$. Then, we have
\begin{align}
\left[\mathbf{0},\mathbf{\bar{L}}\right]
\left[\begin{array}{c}
\mathbf{0}^{\top} \\
\mathbf{\bar{L}^{\top}}
\end{array}\right] &= \mathbf{\bar{L}}\mathbf{\bar{L}^{\top}} \notag \\
&= \left[
\mathbf{a}_{k},\mathbf{L^{(k-1)}} 
\right] \underbrace{\mathbf{G}^{\top} 
\mathbf{G}}_{=\mathbf{I}}
\left[\begin{array}{c}
\mathbf{a}_k^{\top} \\
\big(\mathbf{L}^{(k-1)}\big)^{\top}
\end{array}\right] \notag \\
&= \mathbf{L}^{(k-1)}\big(\mathbf{L}^{(k-1)}\big)^{\top} + \mathbf{a}_{k}\mathbf{a}_k^{\top}
% &= \mathbf{K}^{(k)}
\end{align}
Thus, $\mathbf{\bar{L}}=\mathbf{L}^{(k)}$ is the updated Cholesky factor we need. To rotate $a_{ki}=0$, we construct the Givens matrix as follows:
\begin{equation}
\nonumber
\mathbf{G}_i=\mathbf{I}+(c_i-1)(\mathbf{e}_1\mathbf{e}_1^{\top}+\mathbf{e}_i\mathbf{e}_i^{\top})+s_i(\mathbf{e}_i\mathbf{e}_1^{\top}-\mathbf{e}_1\mathbf{e}_i^{\top}),
\end{equation}
where $c_i = l_{ii}^{(k)}\big/\sqrt{\big(l_{ii}^{(k)}\big)^{2} + a_{ki}^{2}}$,
$s_i = a_{ki}\big/\sqrt{\big(l_{ii}^{(k)}\big)^{2} + a_{ki}^{2}}$, and $\mathbf{e}_i$ denotes a standard unit vector with $e_{ii}=1$.
\begin{remark}
The Cholesky factor requires that the elements on the main diagonal are greater than 0. Thus, if $\mathbf{\bar{L}}_{ii} < 0$, we simply filp $c_{i} \gets -c_{i}$ and
$s_{i} \gets -s_{i}$. Then, $\mathbf{\bar{L}}_{ii}>0$ becomes satisfied.
\end{remark}
\begin{remark}
    The function Chol($\mathbf{M}$) refers to the Cholesky factorization of $\mathbf{M}$ and returns its lower triangular factor $\mathbf{L}$.
\end{remark}
The Online-BLS algorithm is summarized in Algorithm \ref{alg:algorithm}.
\begin{algorithm}[tb]
\renewcommand{\algorithmicrequire}{\textbf{Input:}}
\renewcommand{\algorithmicensure}{\textbf{Output:}}
\caption{Online-BLS Algorithm}
\label{alg:algorithm}
\begin{algorithmic}[1] %[1] enables line numbers
\REQUIRE Stream data $\mathbf{X}=\{(\mathbf{x}_k, \mathbf{y}_k)|k=1,2,\cdots,n\}$, parameters $n_1$, $n_2$, $n_3$, $n_4$, and $\lambda$.
\ENSURE Prediction results ${\mathbf{\hat{y}}_1, \cdots, \mathbf{\hat{y}_n}}$.
\STATE Initial $\mathbf{W}^{(0)}=\mathbf{0}$, $\mathbf{L}^{(0)}=Chol(\lambda\mathbf{I})$
\FOR{$k=1,2,\cdots,n$}
\STATE Receive instance: $\mathbf{x}_k$
\STATE Obtain broad feature $\mathbf{a}_{k}$ via equations \ref{eq:z_i}, \ref{eq:h_j}, and \ref{eq:a_1}
\STATE Predict $\mathbf{\hat{y}}_{k}^{\top}=\mathbf{a}_{k}^{\top}\mathbf{W}^{(k-1)}$ and then reveal $\mathbf{y}_{k}$
\STATE Obtain the updated factor $\mathbf{L}^{(k)}$ by equation \ref{eq:gives}
\STATE Obtain $\mathbf{W}^{(k)}$ by equations \ref{eq:fsk}, \ref{eq:bsk}, and \ref{eq:w^k}.
\ENDFOR
\end{algorithmic}
\end{algorithm}

\subsection{Handing Concept Drift}
Concept drift is a common challenge in data streams, where the joint distribution of features and labels changes over time. Existing BLS algorithms cannot handle concept drift scenarios in online learning. Thanks to the flexibility of our Online-BLS framework, we are able to adopt a simple solution to the concept drift problem, by modifying two lines of code upon Algorithm \ref{alg:algorithm}, which is listed in Algorithm \ref{alg:algorithm-cd}.
\begin{algorithm}[tb]
\renewcommand{\algorithmicrequire}{\textbf{Input:}}
\renewcommand{\algorithmicensure}{\textbf{Output:}}
\caption{Handling Concept Drift}
\label{alg:algorithm-cd}
\begin{algorithmic}[1] %[1] enables line numbers
\REQUIRE Stream data $\mathbf{X}=\{(\mathbf{x}_k, \mathbf{y}_k)|k=1,2,\cdots,n\}$, parameters $n_1$, $n_2$, $n_3$, $n_4$, $\lambda$, and decay factor $\mu$.
\ENSURE Prediction results ${\mathbf{\hat{y}}_1, \cdots, \mathbf{\hat{y}_n}}$.
\STATE Initial $\mathbf{W}^{(0)}=\mathbf{0}$, $\mathbf{P}^{(0)}=\mathbf{0}$
\FOR{$k=1,2,\cdots,n$}
\STATE Receive instance: $\mathbf{x}_k$
\STATE Obtain broad feature $\mathbf{a}_{k}$ via equations \ref{eq:z_i}, \ref{eq:h_j}, and \ref{eq:a_1}
\STATE Predict $\mathbf{\hat{y}}_{k}^{\top}=\mathbf{a}_{k}^{\top}\mathbf{W}^{(k-1)}$ and then reveal $\mathbf{y}_{k}$
\STATE $\mathbf{P}^{(k)}=\mu\mathbf{P}^{(k-1)}+\mathbf{a}_{k}\mathbf{a}_{k}^{\top}$, $\mathbf{L}^{(k)} = Chol(\mathbf{P}^{(k)}+\lambda\mathbf{I})$
\STATE Obtain $\mathbf{W}^{(k)}$ by equations \ref{eq:fsk}, \ref{eq:bsk}, and \ref{eq:w^k}.
\ENDFOR
\end{algorithmic}
\end{algorithm}
The core idea is that learning the time-varying concept with new data is better than with old one. Technically, more attention to new data is encouraged by multiplying history by a decay factor $\mu$.

\section{Theoretical Analysis}
\subsection{Error Bound} 
\label{section:eb}
The following lemmas are first introduced before deriving theoretical errors.
\begin{lemma}
\label{lemma:fs}
The solution $\mathbf{\hat{W}}$ derived via forward substitution satisfies
\begin{equation}
(\mathbf{L}+\mathbf{\Gamma}_{1})\mathbf{\hat{W}}=\mathbf{Y}, \quad |\mathbf{\Gamma}_1| \leq C_m u|\mathbf{L}|+\mathcal{O}(u^{2}),
\end{equation}
where $\mathbf{L}\in\mathbb{R}^{m \times m}$ is a lower triangular matrix, $C_m$ denotes a constant positively correlated with $m$, and $\mathcal{O}(u^{2})$ represents the error tail term.
\end{lemma}
The detailed proof can be referred to~\cite{higham2002accuracy}. Lemma \ref{lemma:fs} states that its solution satisfies a slightly perturbed system. Also, each entry in the perturbation matrix $\mathbf{\Gamma}_1$ is much smaller than the corresponding element in $\mathbf{L}$.

\begin{remark}[Definition of $u$]
$\forall x\in\mathbb{R}$, its floating-point representation is given by $fl(x)=
x(1+\gamma)$, where $\gamma$ is the floating-point error and its upper bound, unit roundoff $u$, (i.e., $|\gamma|\leq u$) can be defined as
\begin{equation}
    u=\frac{1}{2}\times(\text{gap between}\,1\,\text{and next largest floating point number}). \notag
\end{equation}
To be specific, the unit roundoff for IEEE single format is about $10^{-7}$ and
for double format is about $10^{-16}$.
\end{remark}
\begin{lemma}
\label{lemma:bs}
The solution $\mathbf{\hat{W}}$ derived via backward substitution satisfies
\begin{equation}
(\mathbf{L}^{\top}+\mathbf{\Gamma}_2)\mathbf{\hat{W}}=\mathbf{Y}, \quad |\mathbf{\Gamma}_2|\leq C_m u|\mathbf{L}^{\top}|+\mathcal{O}(u^2),
\end{equation}
where $\mathbf{L}^{\top}$ is an upper triangular matrix.
\end{lemma}
The proof also refers to ~\cite{higham2002accuracy}.

\noindent\textbf{Error Bound of Online-BLS}. Recall that the incremental weight $\mathbf{\Delta}\mathbf{W}^{(k)}$ in Online-BLS is solved via equations \ref{eq:spk1} and \ref{eq:spk2}. Therefore, the error bound is given in Theorem \ref{theorem:obls}. For brevity, we omit the superscript $(k)$ in the following.
\begin{theorem}
\label{theorem:obls}
Let $\mathbf{\hat{L}}$ be the estimated Cholesky factor $\mathbf{L}$, we have
\begin{equation}
\mathbf{\hat{L}}\mathbf{\hat{L}}^{\top}\mathbf{\Delta}\mathbf{\hat{W}}=\big(\mathbf{L}\mathbf{L}^{\top}+\mathbf{\Gamma}\big)\mathbf{\Delta}\mathbf{\hat{W}}=\mathbf{B} \notag
\end{equation}
with
\begin{align}
|\mathbf{B}-\mathbf{L}\mathbf{L}^{\top}\mathbf{\Delta}\mathbf{\hat{W}}|
&=|\mathbf{\Gamma}||\mathbf{\Delta}\mathbf{\hat{W}}| \notag \\
\label{eq:eb_cd}
&\leq C_m u|\mathbf{L}||\mathbf{L}^{\top}||\mathbf{\Delta}\mathbf{\hat{W}}|+\mathcal{O}(u^{2}).
\end{align}
\end{theorem}
\begin{proof}
From Lemmas \ref{lemma:fs} and \ref{lemma:bs}, we have
\begin{align}
&(\mathbf{L}+\mathbf{\Gamma}_{1})\mathbf{\hat{C}}=\mathbf{B}, \hspace{1.7em} |\mathbf{\Gamma}_1| \leq C_m u|\mathbf{L}|+\mathcal{O}(u^{2}), \notag \\
&(\mathbf{L}^{\top}+\mathbf{\Gamma}_{2})\mathbf{\Delta}\mathbf{\hat{W}}=\mathbf{\hat{C}}, \quad|\mathbf{\Gamma}_2| \leq C_m u|\mathbf{L}^{\top}|+\mathcal{O}(u^{2}), \notag
\end{align}
and thus 
\begin{align}
\mathbf{B} &= (\mathbf{L}+\mathbf{\Gamma}_{1})(\mathbf{L}^{\top}+\mathbf{\Gamma}_{2})\mathbf{\Delta}\mathbf{\hat{W}} \notag \\
&= (\mathbf{L}\mathbf{L}^{\top}+\mathbf{L}\mathbf{\Gamma}_2+\mathbf{\Gamma}_1\mathbf{L}^{\top}+\mathbf{\Gamma}_1\mathbf{\Gamma}_2)\mathbf{\Delta}\mathbf{\hat{W}}. \notag
\end{align}
Thus, we find $(\mathbf{L}\mathbf{L}^{\top}+\mathbf{\Gamma})\mathbf{\Delta}\mathbf{\hat{W}}=\mathbf{B}$ with
\begin{equation}
|\mathbf{\Gamma}| 
\leq |\mathbf{L}||\mathbf{\Gamma}_2|+|\mathbf{\Gamma}_1||\mathbf{L}^{\top}|+\mathcal{O}(u^{2}) \notag \\
\leq C_m u|\mathbf{L}||\mathbf{L}^{\top}|+\mathcal{O}(u^{2}), \notag
\end{equation}
such that
\begin{align}
|\mathbf{B}-\mathbf{L}\mathbf{L}^{\top}\mathbf{\Delta}\mathbf{\hat{W}}|&=|\mathbf{\Gamma}||\mathbf{\Delta}\mathbf{\hat{W}}| \notag \\
&\leq C_m u|\mathbf{L}||\mathbf{L}^{\top}||\mathbf{\Delta}\mathbf{\hat{W}}|+\mathcal{O}(u^2). \notag
\end{align}
Thus, we can conclude equation \ref{eq:eb_cd}.
\end{proof}

\noindent\textbf{Error Bound of Inverse Methods}. As for the previous I-BLS, BLS-CIL, and RI-BLS, none of them avoids computing the inverse of a large matrix, which means to obtain $\mathbf{\Delta}\mathbf{W}=fl\left((\mathbf{A}^{\top}\mathbf{A}+\lambda\mathbf{I})^{-1}\mathbf{B}\right)$ only with rounding errors. Then, we have
\begin{align}
\mathbf{\Delta}\mathbf{\hat{W}}=\big((\mathbf{A}^{\top}\mathbf{A}+&\lambda\mathbf{I})^{-1}+\mathbf{\Gamma}\big)\mathbf{B}, \notag \\
&|\mathbf{\Gamma}| \leq C_m u|(\mathbf{A}^{\top}\mathbf{A}+\lambda\mathbf{I})^{-1}|+\mathcal{O}(u^2) \notag
\end{align} and
$(\mathbf{A}^{\top}\mathbf{A}+\lambda\mathbf{I})\mathbf{\Delta}\mathbf{\hat{W}}=\mathbf{B}+(\mathbf{A}^{\top}\mathbf{A}+\lambda\mathbf{I})\mathbf{\Gamma}\mathbf{B}$.
Thus, its error bound is
\begin{align}
|\mathbf{B}-&(\mathbf{A}^{\top}\mathbf{A}+\lambda\mathbf{I})\mathbf{\Delta}\mathbf{\hat{W}}| \notag \\
&\leq C_m u |(\mathbf{A}^{\top}\mathbf{A}+\lambda\mathbf{I})||(\mathbf{A}^{\top}\mathbf{A}+\lambda\mathbf{I})^{-1}||\mathbf{B}|+\mathcal{O}(u^2) \notag \\
\label{eq:eb_im}
&\leq C_m u |\mathbf{L}||\mathbf{L}^{\top}||(\mathbf{A}^{\top}\mathbf{A}+\lambda\mathbf{I})^{-1}||\mathbf{B}|+\mathcal{O}(u^2)
\end{align}

Observing equations \ref{eq:eb_cd} and \ref{eq:eb_im}, since $\mathbf{A}$ is generally an ill-conditioned matrix for BLS, there exists $|\mathbf{\Delta}\mathbf{\hat{W}}|\ll|(\mathbf{A}^{\top}\mathbf{A}+\lambda\mathbf{I})^{-1}||\mathbf{B}|$. Thus, matrix inverse methods provide higher error bounds than ours.

\subsection{Time Complexity}
\label{section:tc}
The time complexities of our Online-BLS and three existing incremental BLS are shown in Table \ref{tab:tc}.
\begin{table}
\centering
\begin{tabular}{ll}
\toprule
Method & Time Complexity \\
\midrule
I-BLS & $\mathcal{O}(km+m^2+m^3+mc)$/$\mathcal{O}(km+m+mc)$ \\
BLS-CIL & $\mathcal{O}(m^3+m^2c+m^2+mc+m)$\\
RI-BLS & $\mathcal{O}(m^3+m^2c+m^2+mc)$ \\
\midrule
Ours & $\mathcal{O}(m^2c+m^2+mc)$ \\
\bottomrule
\end{tabular}
\caption{Time complexity}
\label{tab:tc}
\end{table}
First, the time complexity of I-BLS contains the primary term of $k$. That is, as the number of samples received by I-BLS increases, its online update time overhead increases linearly. This is quite terrible for online learning tasks, which tend to have a lot or even an infinite number of samples. By comparing the time complexity of ours and the remaining two methods, we find that RI-BLS and BLS-CIL have one (i.e., $m^3$) and two (i.e., $m^3$ and $m$) additional terms over Online-BLS, respectively. Thus, our Online-BLS algorithm is remarkably efficient and well suitable for online machine learning tasks.

\section{Experiments}
The most common prequential test-then-train experimental paradigm is adopted. First, Online-BLS is compared with state-of-the-art incremental BLS algorithms on the six real-world stationary datasets. Then, parameter sensitivity analysis experiments are performed to verify the stability of Online-BLS. Third, ablation experiments are performed by removing EOUS from our method. Finally, additional experiments are performed on the four non-stationary datasets to demonstrate the superiority of our approach over state-of-the-art baselines. Table \ref{tab:dataset} and Appendix B summarize and detail the dataset used. The metrics used in this paper are detailed in Appendix C. All experiments are conducted on the Ubuntu 20.04 operating system, and the CPU is AMD EPYC 7302. 
\begin{table}
\centering
\begin{tabular}{lrrrr}
\toprule
Dataset & IA & C & DP & Type \\
\midrule
IS & 19 & 7 & 2,310 & Stationary \\
USPS & 256 & 10 & 9,298 & Stationary \\
Letter & 16 & 26 & 20,000 & Stationary \\
Adult & 14 & 2 & 45,222 & Stationary \\
Shuttle & 8 & 7 & 58,000 & Stationary \\
MNIST & 784 & 10 & 70,000 & Stationary \\
Hyperplane & 20 & 2 & 100,000 & Concept Drift \\
SEA & 3 & 2 & 100,000 & Concept Drift \\
Electricity & 6 & 2 & 45,312 & Concept Drift \\
CoverType & 54 & 7 & 581,012 & Concept Drift \\
\bottomrule
\end{tabular}
\caption{Descriptions of datasets. IS: Image Segment; IA: Input Attributes; C: Classes; DP: Data Points.}
\label{tab:dataset}
\end{table}

\subsection{Experiments on Stationary Datasets}
\label{section:cweobls}
The comparison methods include I-BLS~\cite{7987745}, BLS-CIL~\cite{10086560}, and RI-BLS~\cite{10533441}. 
To ensure fairness, we set the parameters common to all comparison methods and Online-BLS to be the same. Specifically, $n_1$, $n_2$, $n_3$, and $n_4$ were set to $10$, $10$, $1000$, and $1$, respectively. The regularization parameter $\lambda$ was set to $1e-8$ for I-BLS, RI-BLS and Online-BLS. For BLS-CIL, $\lambda_1$ and $\lambda_2$ were set to $1$ in the original paper, which is unstable in the online learning task. Therefore, we tuned both to 0.1 to obtain better performance.
To draw credible conclusions, we repeated each experiment 10 times by changing the order of streaming data and the initialization parameters.

The average and Standard Deviation (SD) of the final Online Cumulative Accuracies (OCA) for all methods are shown in Table \ref{tab:oca}. 
\begin{table*}
\centering
\begin{tabular}{lcccccc}
\toprule
Method & IS & USPS & Letter & Adult & Shuttle & MNIST \\
\midrule
I-BLS & 73.6$\pm$19.80 & 61.1$\pm$40.10 & 59.4$\pm$27.30 & 71.3$\pm$13.60 & 86.6$\pm$20.10 & 83.5$\pm$24.40 \\
BLS-CIL & 83.6$\pm$0.423 & \textbf{93.7}\bm{$\pm$}\textbf{0.239} & 79.7$\pm$0.711 & \underline{75.4$\pm$0.020} & \underline{95.9$\pm$0.299} & \underline{89.9$\pm$0.180} \\
RI-BLS  & 89.9$\pm$0.668 & 92.8$\pm$0.228 & \underline{87.6$\pm$0.541} & 72.0$\pm$0.856 & 90.2$\pm$1.100 & 89.9$\pm$0.246 \\
\midrule
Ours w/o EOUS & \textbf{90.8}\bm{$\pm$}\textbf{0.229} & \underline{93.6$\pm$0.386} & \textbf{88.9}\bm{$\pm$}\textbf{0.137} & \textbf{76.4}\bm{$\pm$}\textbf{0.050} & \textbf{98.2}\bm{$\pm$}\textbf{0.026} & \textbf{92.5}$\bm{\pm}$\textbf{0.105} \\
Ours & \textbf{90.8}$\bm{\pm}$\textbf{0.229} & \underline{93.6$\pm$0.386} & \textbf{88.9}$\bm{\pm}$\textbf{0.137} & \textbf{76.4}$\bm{\pm}$\textbf{0.051} & \textbf{98.2}$\bm{\pm}$\textbf{0.026} & \textbf{92.5}$\bm{\pm}$\textbf{0.105} \\
\bottomrule
\end{tabular}
\caption{Average and SD of OCA (\%). Note: Bold indicates the best result, and underlining denotes the second-best result.}
\label{tab:oca}
\end{table*}
First, our proposed Online-BLS exceeds all the baselines on 5 out of 6 datasets. For the USPS dataset, the OCA of Online-BLS is also very close to the highest one. Also, our Online-BLS has low standard deviations, which means with excellent stability for different experimental trials.
The average and SD of the online update time are shown in Figure \ref{fig:tt}.
\begin{figure}[tb]
\centering
\subfloat[IS]{
\includegraphics[width=1.08in, height=0.78in]{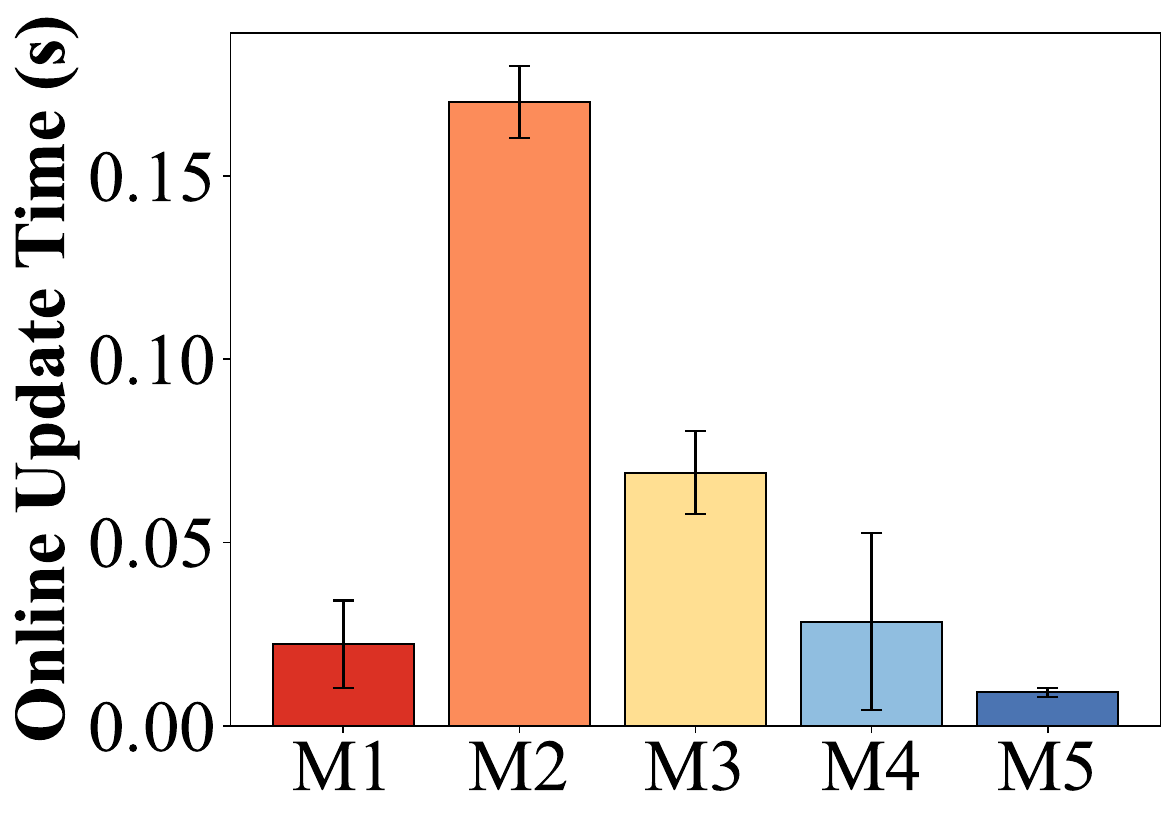}
\label{tt:sf1}
}
\subfloat[USPS]{
\includegraphics[width=1.08in, height=0.78in]{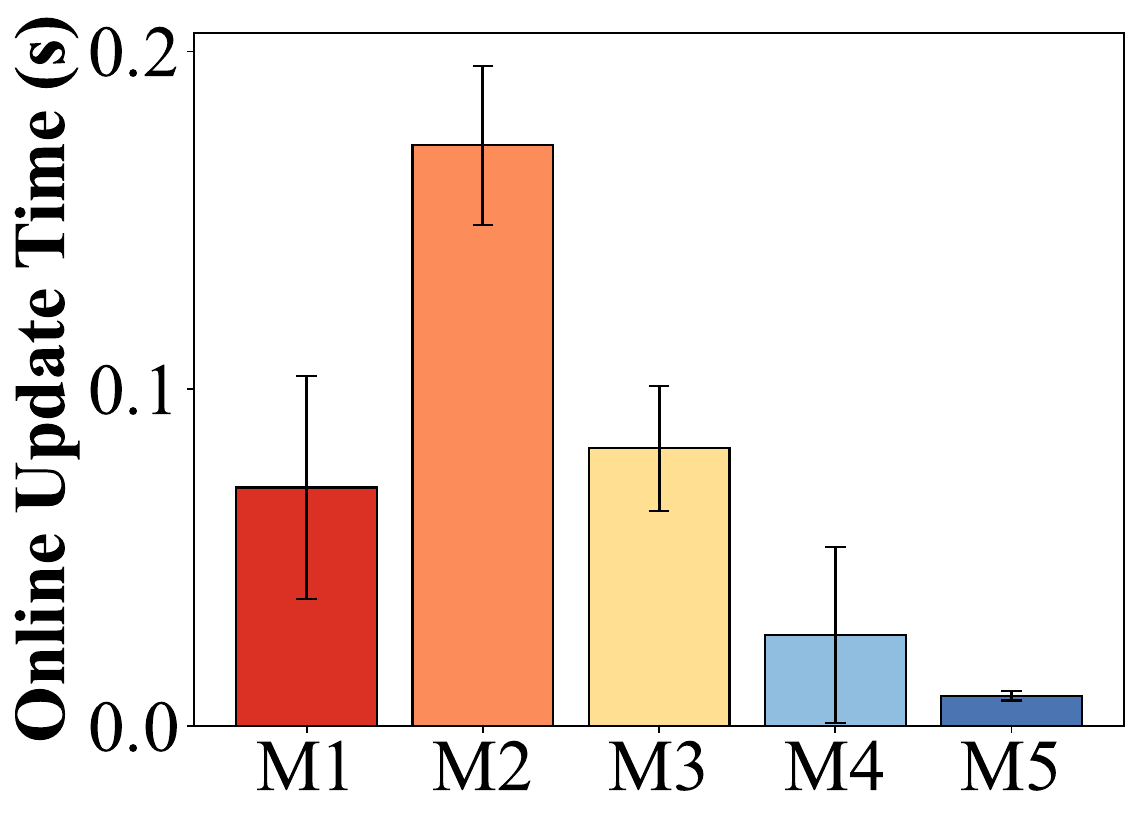}
\label{tt:sf2}
}
\subfloat[Letter]{
\includegraphics[width=1.08in, height=0.78in]{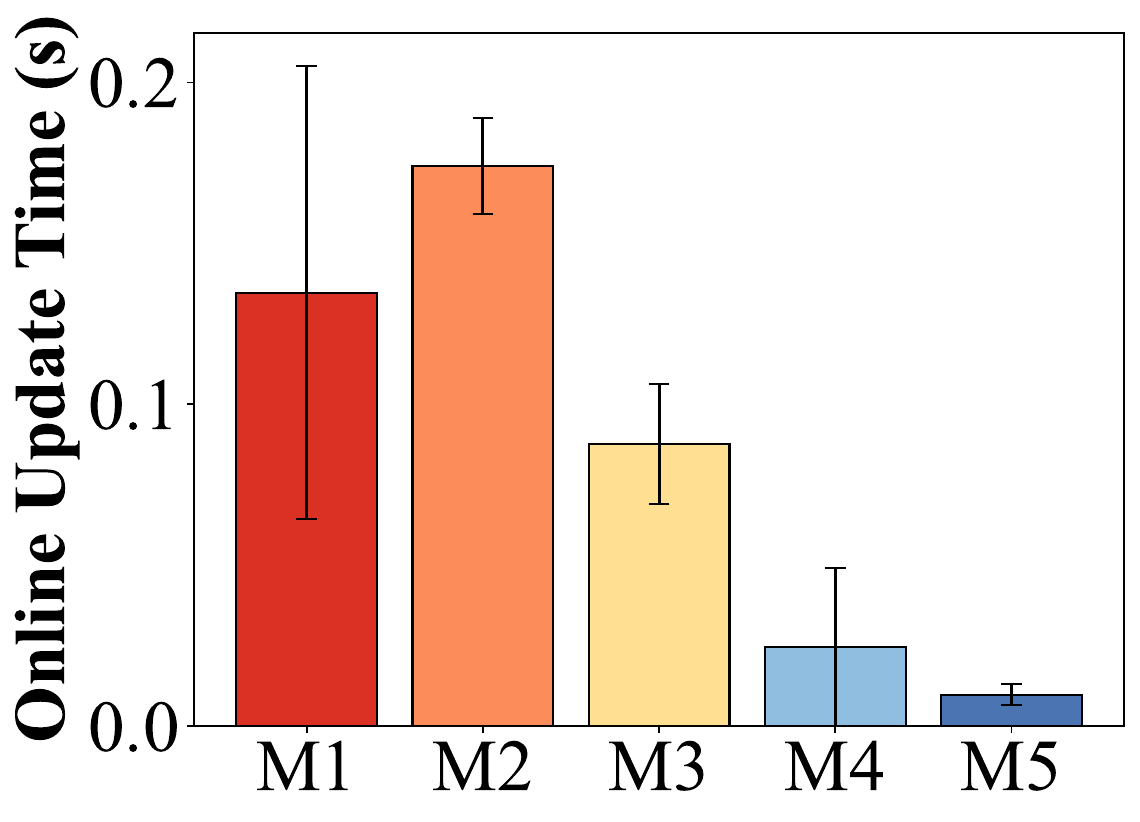}
\label{tt:sf3}
}
\newline
\subfloat[Adult]{
\includegraphics[width=1.08in, height=0.78in]{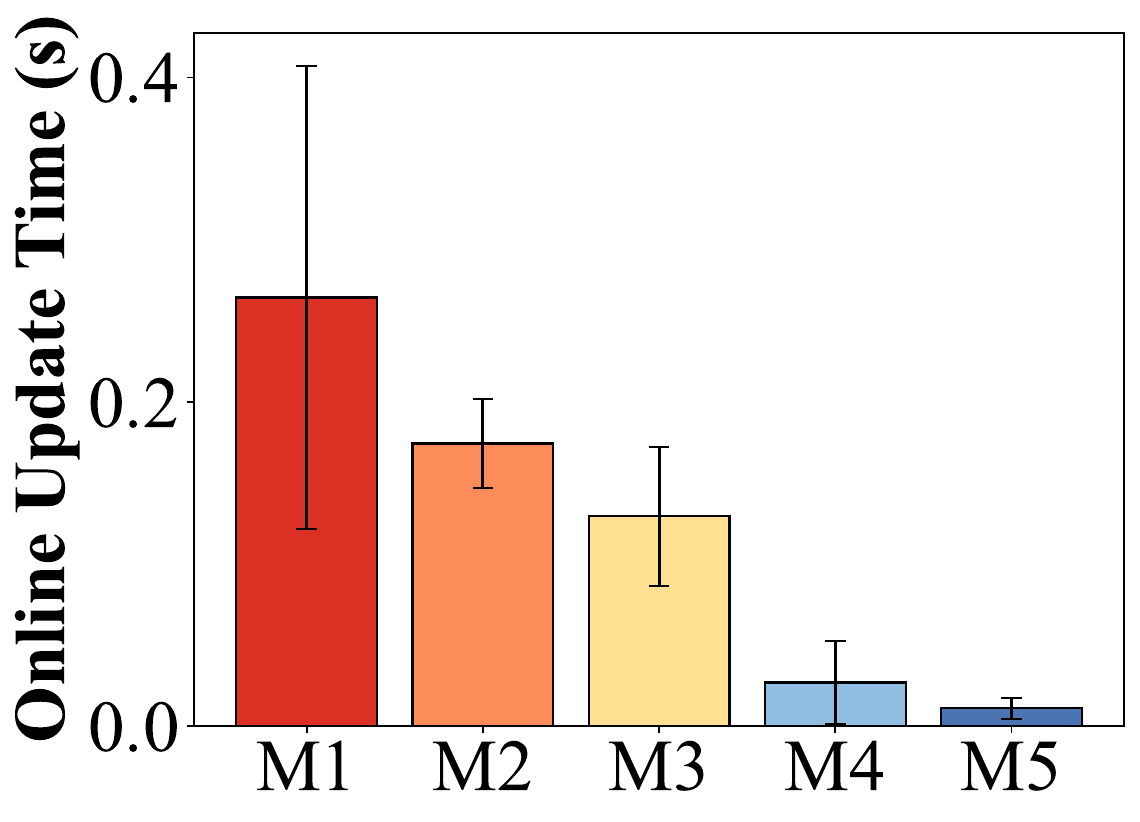}
\label{tt:sf4}
}
\subfloat[Shuttle]{
\includegraphics[width=1.08in, height=0.78in]{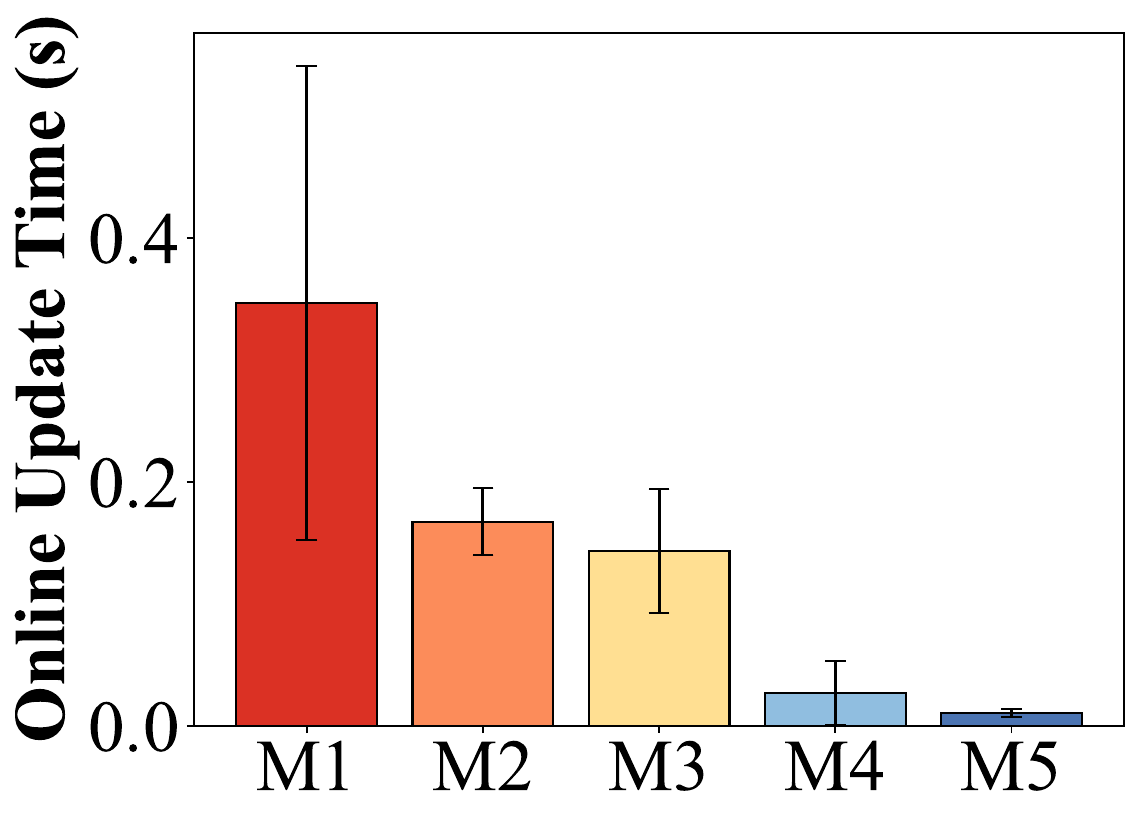}
\label{tt:sf5}
}
\subfloat[MNIST]{
\includegraphics[width=1.08in, height=0.78in]{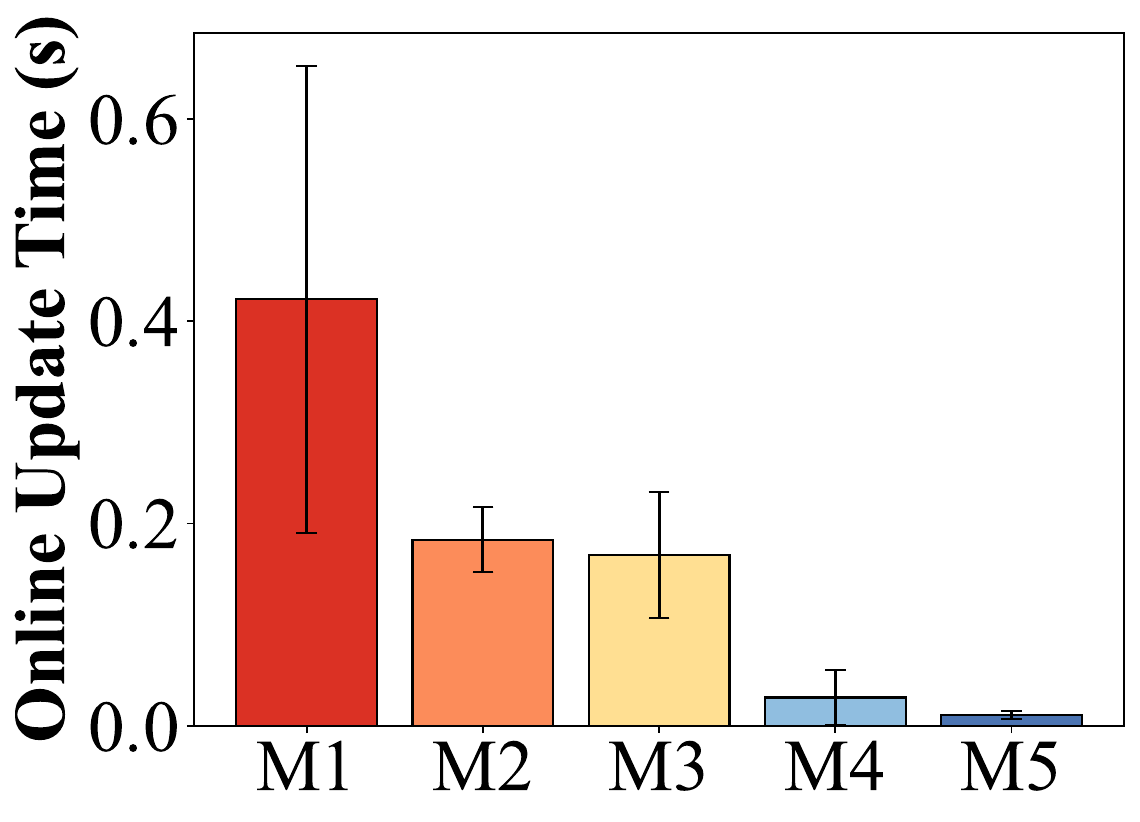}
\label{tt:sf6}
}
\caption{Average and Standard Deviation (SD) of online update time. M1: I-BLS; M2: BLS-CIL; M3: RI-BLS; M4: Ours w/o EOUS; M5: Ours.}
\label{fig:tt}
\end{figure}% \begin{figure*}[tb]
First, our Online-BLS has the shortest online update time, indicating that our method is very efficient. Second, we find that the standard deviation of the online update time of I-BLS increases with the number of instances. These results are also consistent with the time complexity analysis in Section \ref{section:tc}.
Figure \ref{fig:oca} shows the convergence curves of all the methods.
\begin{figure}[tb]
\centering
\subfloat[IS]{
\includegraphics[width=1.08in, height=0.78in]{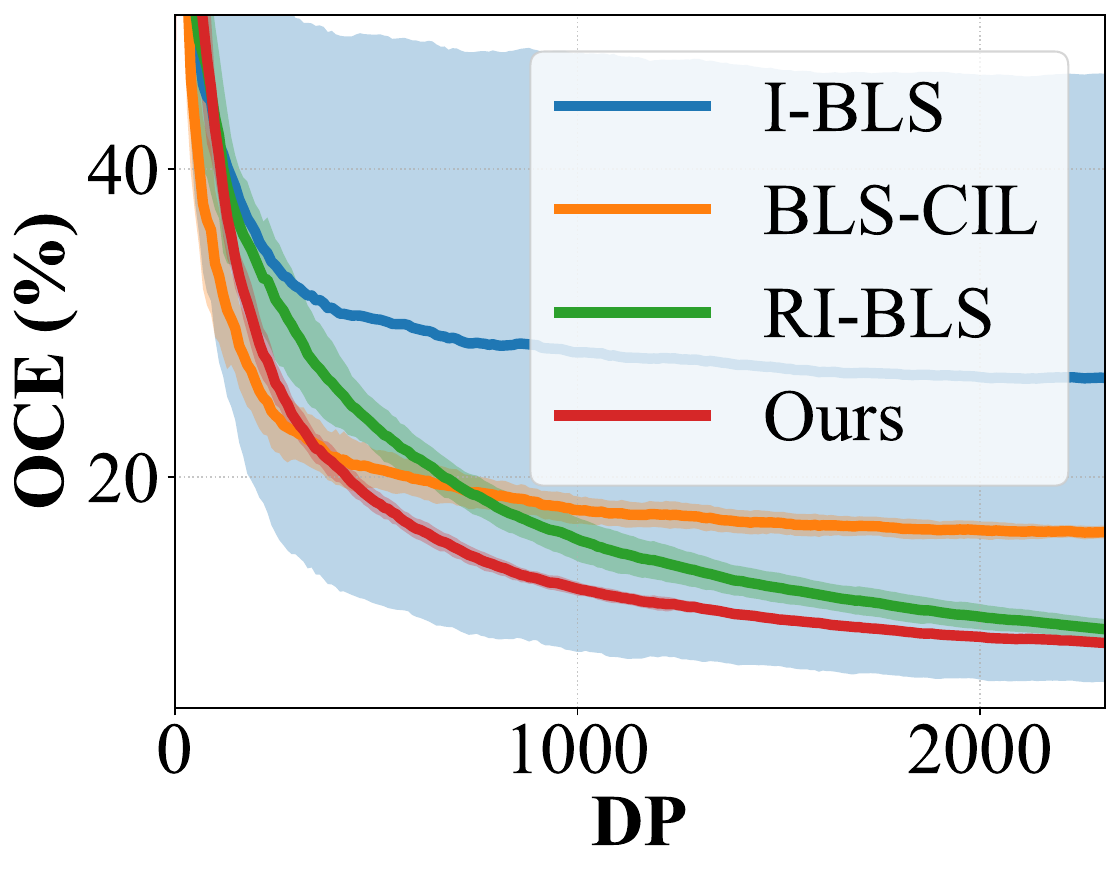}
\label{sf1}
}
\subfloat[USPS]{
\includegraphics[width=1.08in, height=0.8in]{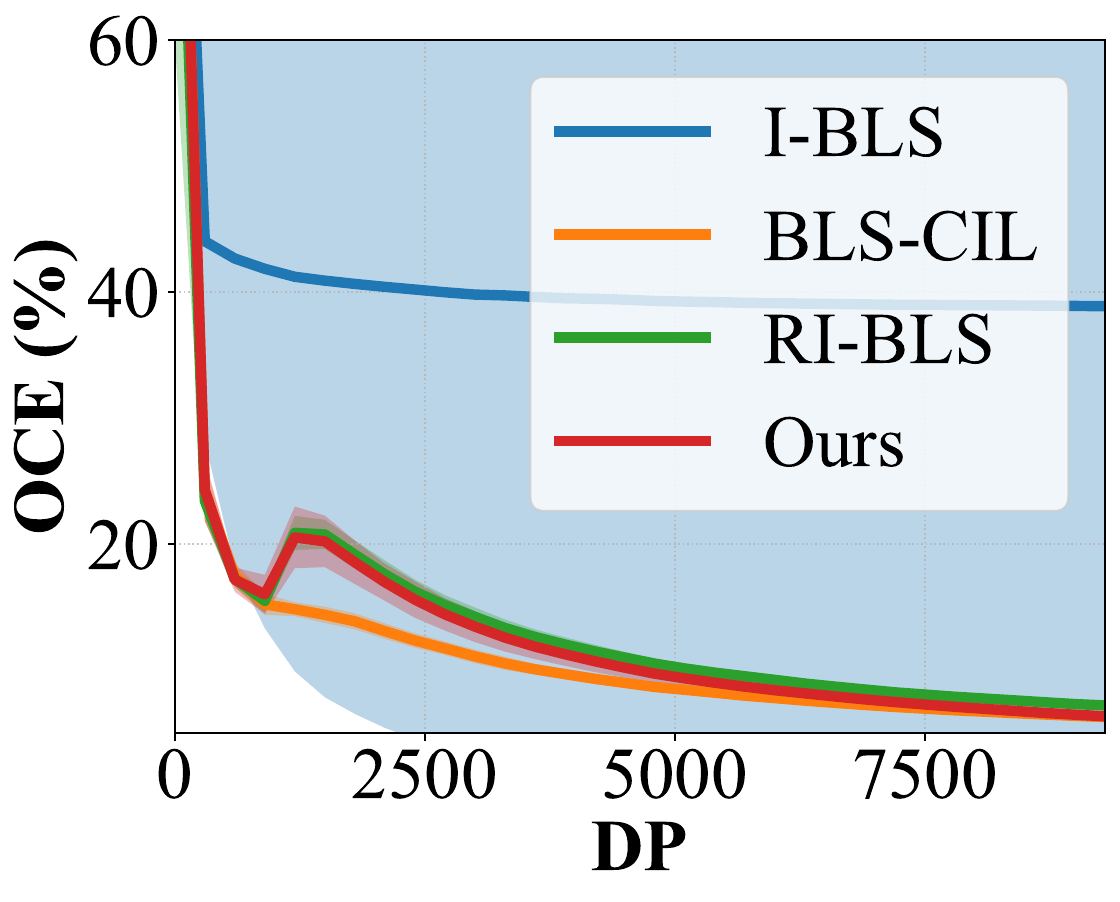}
\label{sf2}
}
\subfloat[Letter]{
\includegraphics[width=1.17in, height=0.8in]{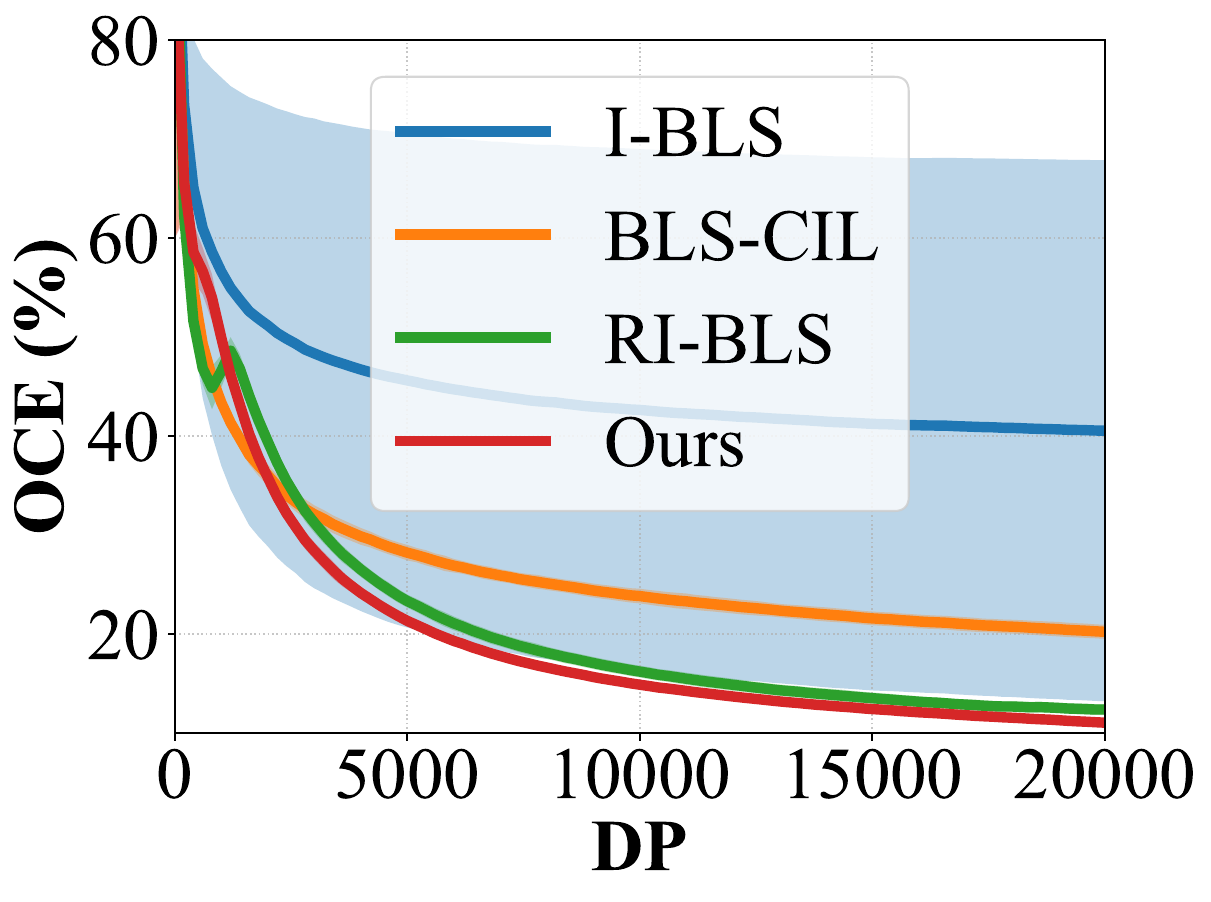}
\label{sf3}
}
\newline
\subfloat[Adult]{
\includegraphics[width=1.08in, height=0.8in]{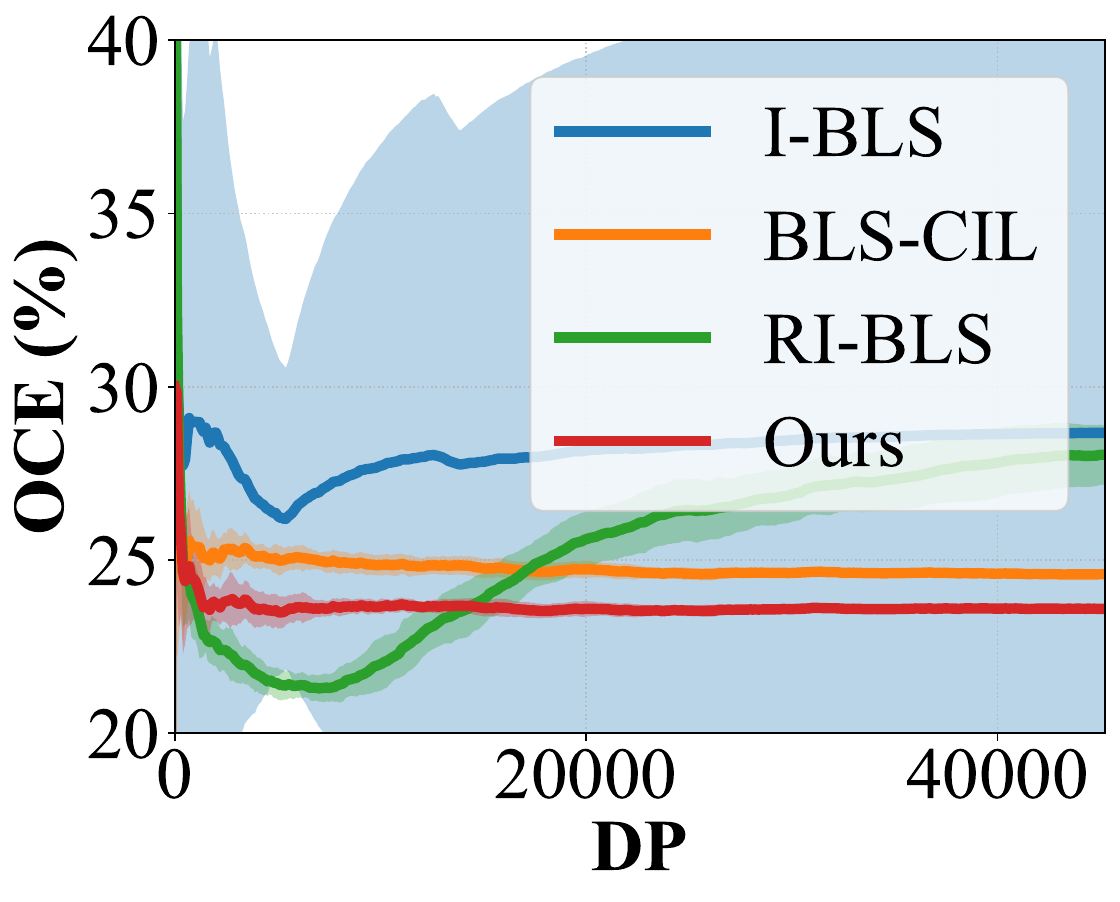}
\label{sf4}
}
\subfloat[Shuttle]{
\includegraphics[width=1.08in, height=0.8in]{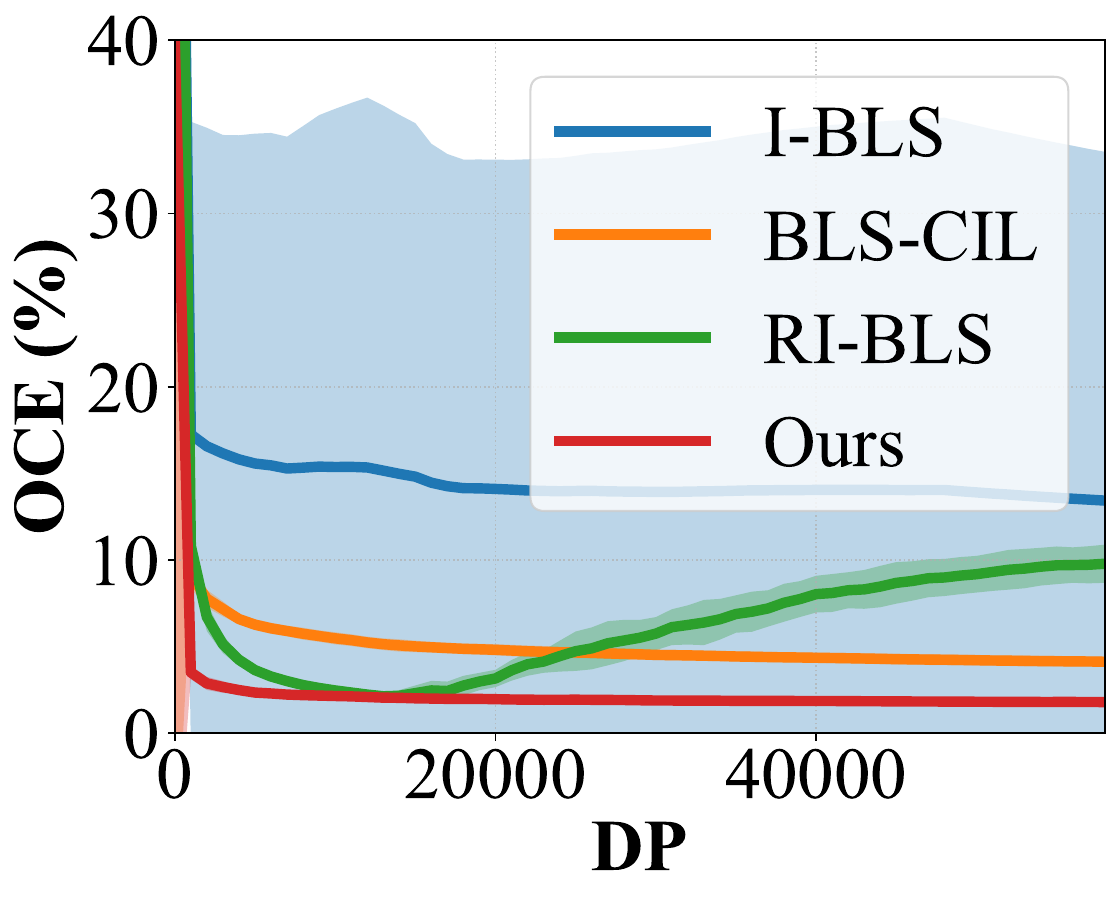}
\label{sf5}
}
\subfloat[MNIST]{
\includegraphics[width=1.08in, height=0.78in]{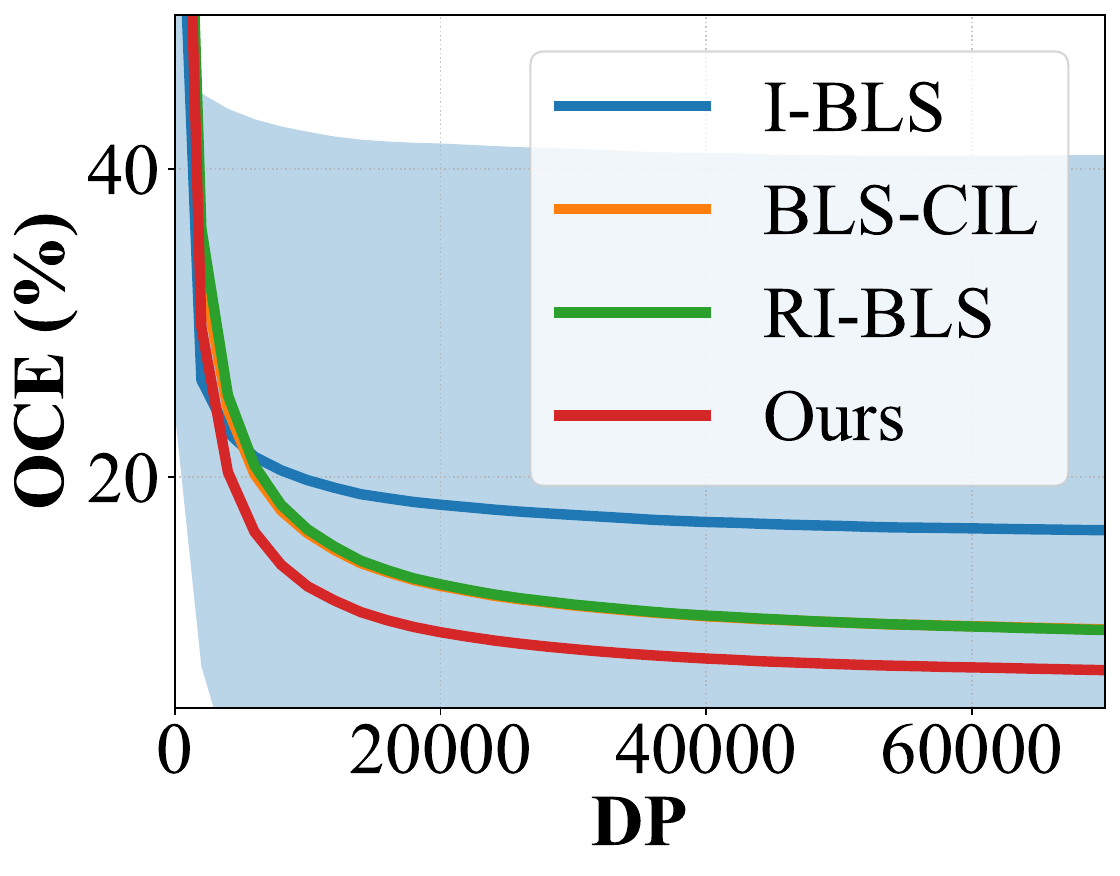}
\label{sf6}
}
\caption{Convergence curves of average and SD. Ours w/o EOUS is omitted because its results are consistent with those of Ours.}
\label{fig:oca}
\end{figure}
The sum of the online cumulative error (OCE) and OCA is one.
Online-BLS consistently converges rapidly and outperforms all the baselines at almost all time steps on all datasets. It also reaffirms the effectiveness of our method.

\subsection{Parameter Sensitivity Analysis}
Our proposed Online-BLS comes with several tuning parameters, including the enhancement nodes $n_3$. The impact of different $n_3$ on the average OCE is shown in Figure \ref{fig:n3}. First, our method consistently achieves the lowest average OCE compared to all baselines on all $n_3$ cases across the six datasets. Second, we find that the average OCE of I-BLS changes significantly with different $n_3$. This is because I-BLS is very unstable. That is, the accuracy of I-BLS in some trials is close to that of random guessing cases. Based on the above results, we can safely conclude that Online-BLS is robust to changes of $n_3$. Hence, we simply set $n_3$ to $1,000$ on all datasets and did not adjust it for different datasets.
\begin{figure}[tb]
\centering
\subfloat[IS]{
\includegraphics[width=1.08in, height=0.8in]{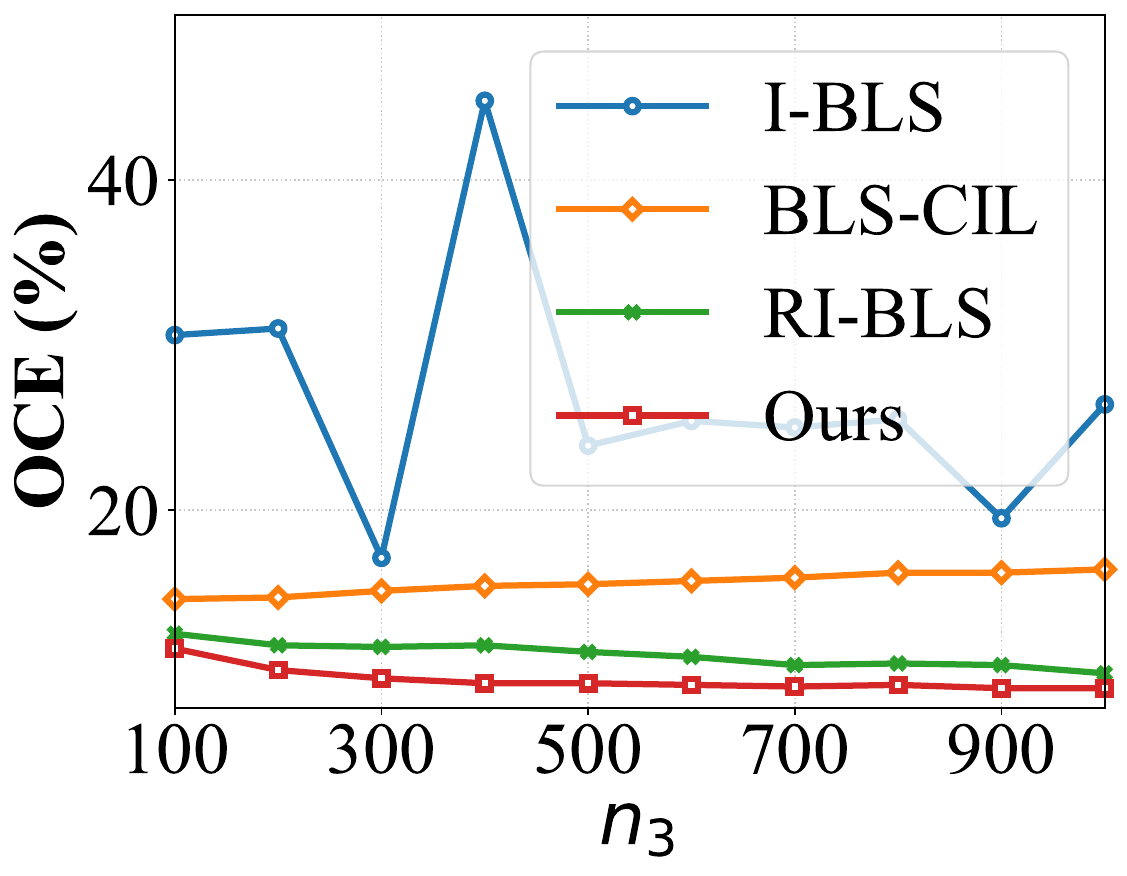}
\label{n3:sf1}
}
\subfloat[USPS]{
\includegraphics[width=1.08in, height=0.8in]{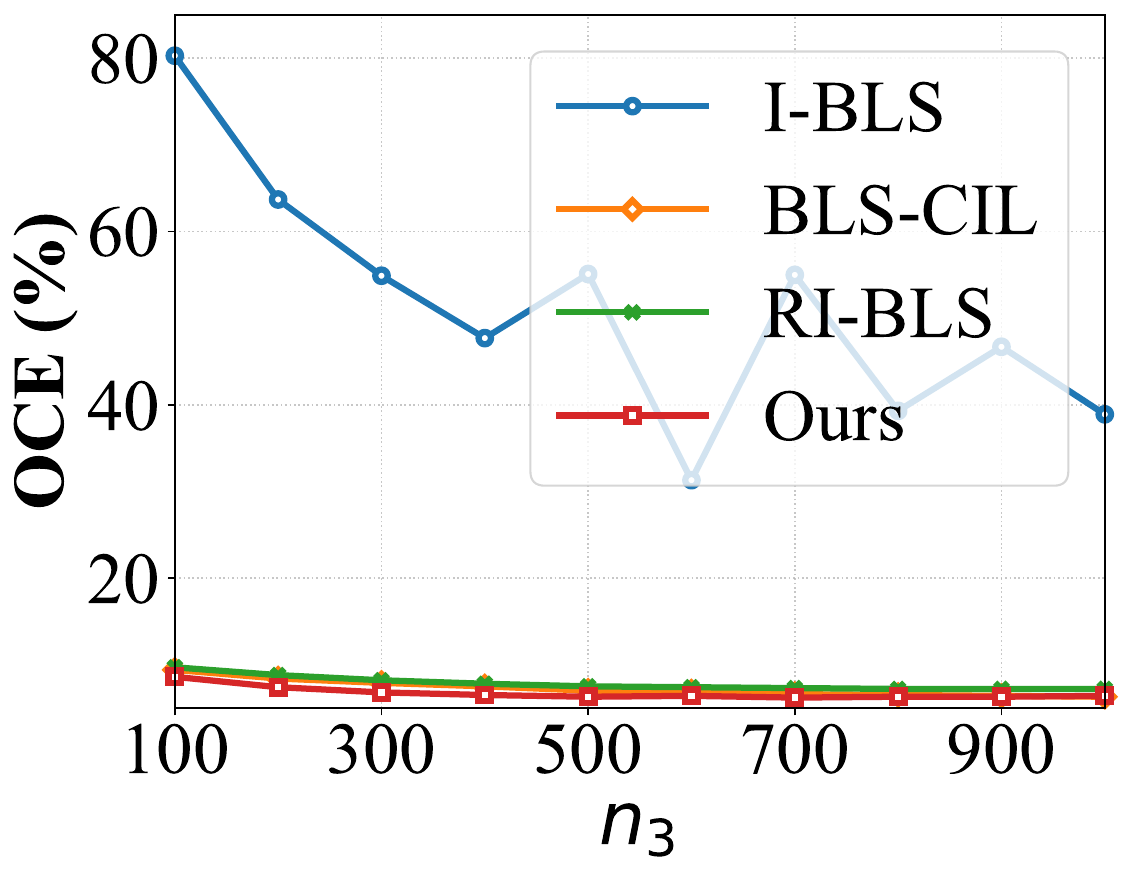}
\label{n3:sf2}
}
\subfloat[Letter]{
\includegraphics[width=1.08in, height=0.8in]{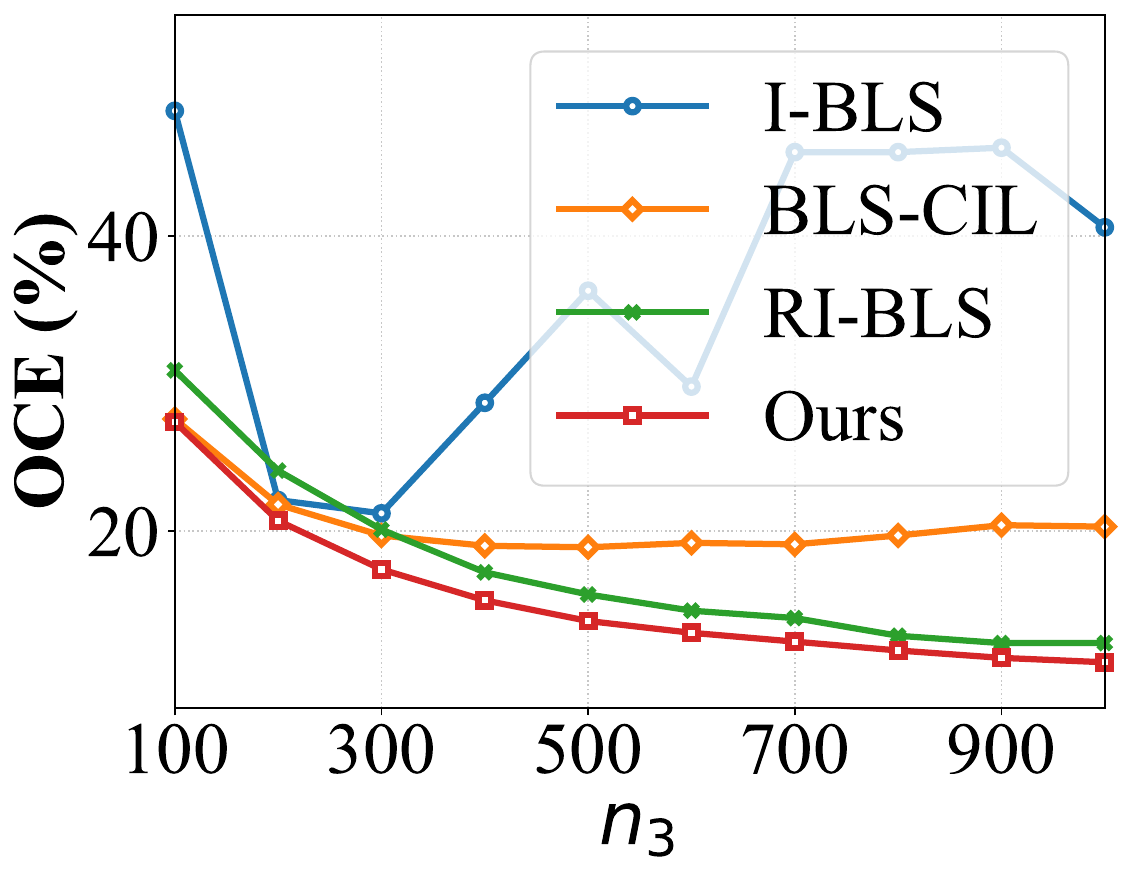}
\label{n3:sf3}
}
\newline
\subfloat[Adult]{
\includegraphics[width=1.08in, height=0.8in]{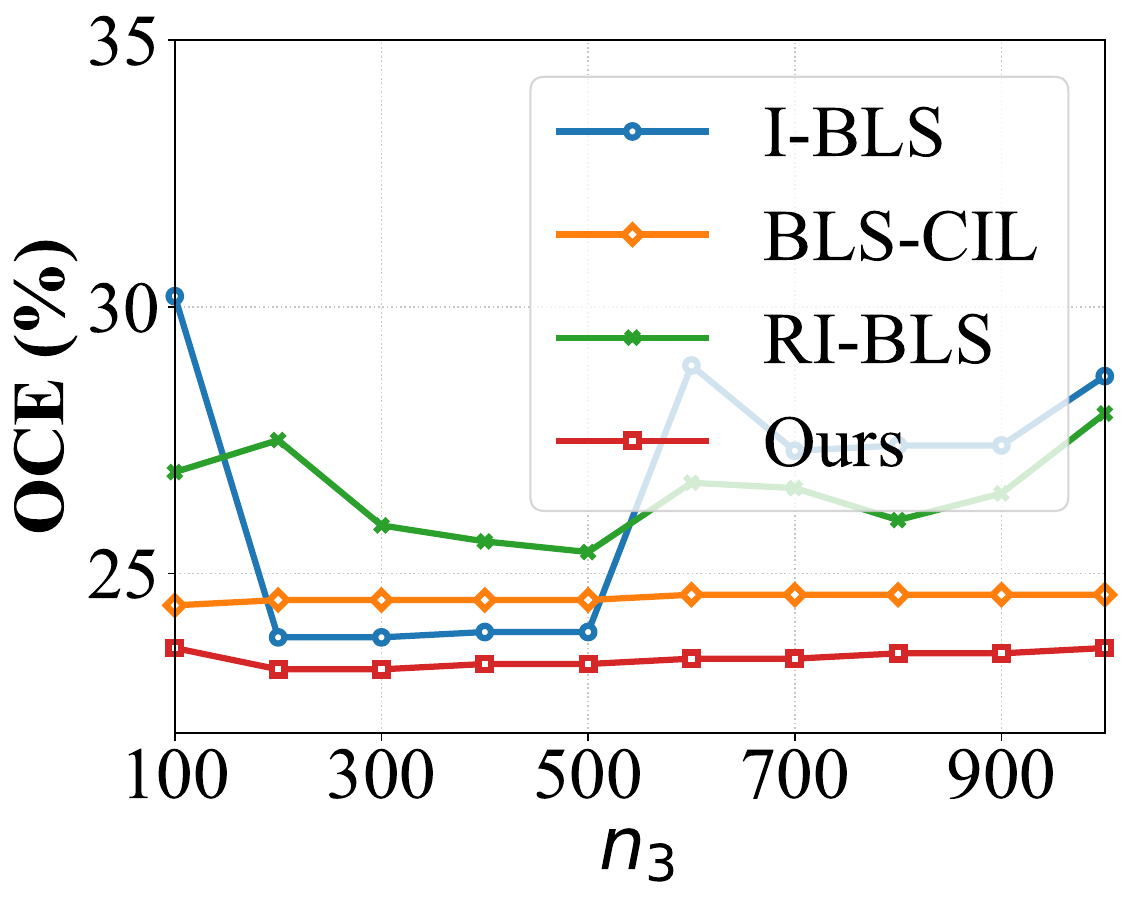}
\label{n3:sf4}
}
\subfloat[Shuttle]{
\includegraphics[width=1.08in, height=0.8in]{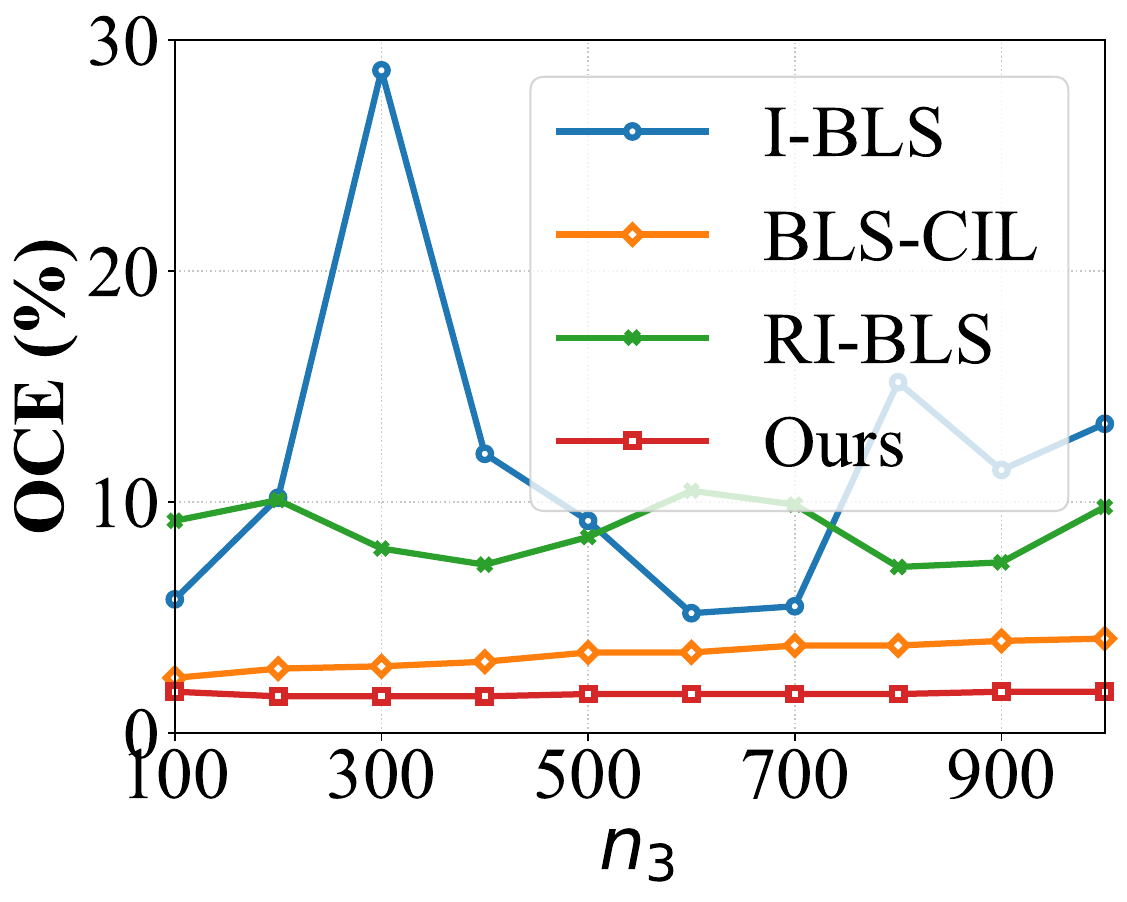}
\label{n3:sf5}
}
\subfloat[MNIST]{
\includegraphics[width=1.08in, height=0.8in]{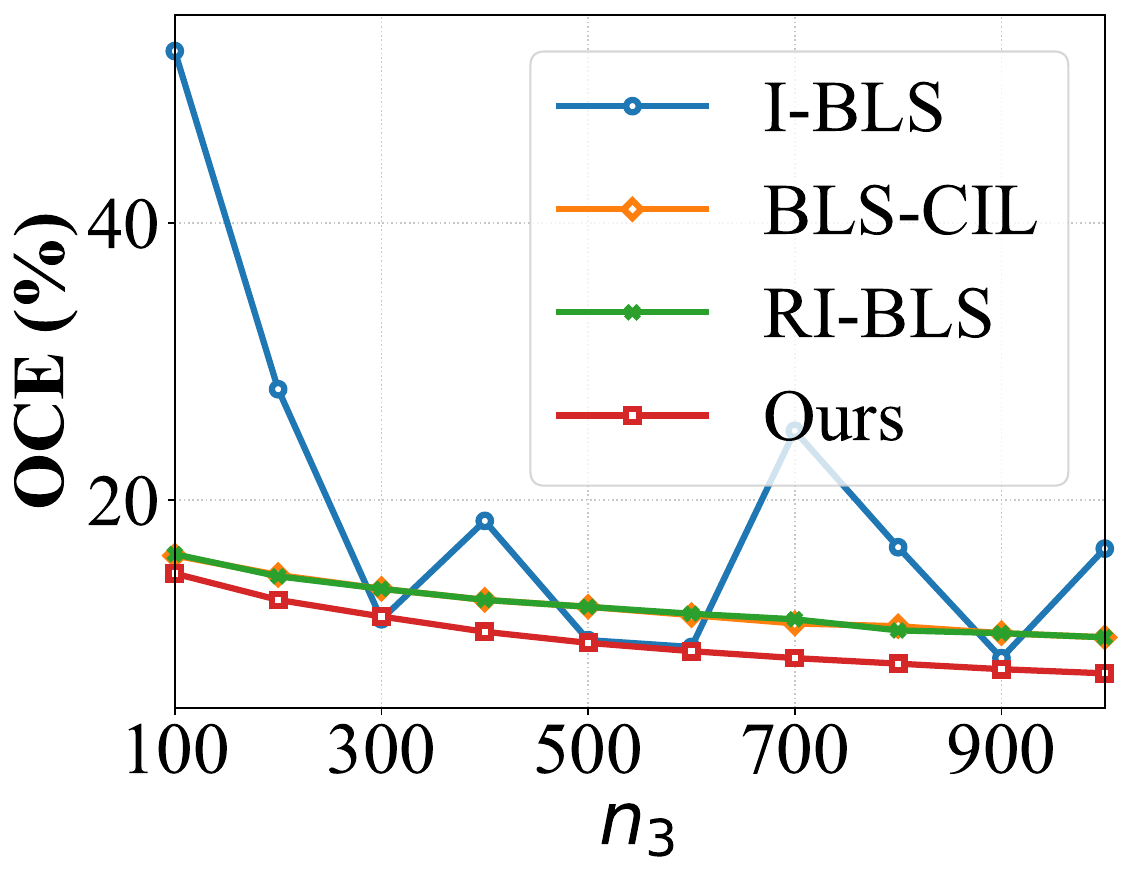}
\label{n3:sf6}
}
\caption{Average OCE under different $n_3$ values. Ours w/o EOUS is omitted because its results are consistent with those of Ours.}
\label{fig:n3}
\end{figure}

\subsection{Ablation Study}
Our framework has two main parts, EWEA and EOUS. Amongst them, EWEA is an inseparable part of Online-BLS. Thus, we examine the validity of EOUS in this section. From Table \ref{tab:oca} and Figure \ref{fig:tt}, we can find that EOUS can significantly reduce online update time while maintaining model accuracy. Hence, we can conclude that EOUS is an efficient algorithm that can significantly reduce the online update overhead of Online-BLS.

\subsection{Experiments with Concept Drift Dataset}
In this section, the effectiveness and flexibility of our Online-BLS are validated in the concept drift scenario. To learn non-stationary data streams, we adopt a forgetting factor to force our Online-BLS to assign larger weights to new arrivals. Specifically, we simply set it to $0.99$ for all datasets.
We compare our method with two classes of baselines used for data stream classification with concept drift. The first category is online machine learning algorithms, including HT~\cite{10.1145/502512.502529} and ARF~\cite{gomes2017adaptive}. The second one is SOTA online deep learning models such as HBP~\cite{sahoo2018online}, ADL~\cite{ashfahani2019autonomous}, CAND~\cite{gunasekara2022online}, ATNN~\cite{wen2024adaptive}, and EODL~\cite{su2024elastic}.
Beyond OCA, four metrics are used. They are Balanced Accuracy (BACC), Average Balanced Accuracy (AVRBACC), F1 score (F1), and Matthews Correlation Coefficient (MCC). It is worth noting that all the metrics used in this paper, except OCE, are such that larger values represent better model performance. The details of these metrics can also be found in Appendix C. The parameter settings of our method are consistent with those in section \ref{section:cweobls}. Tables \ref{tab:hp}, \ref{tab:sea}, \ref{tab:elec}, and \ref{tab:coty} demonstrate the final results on the Hyperplane, SEA, Electricity, and CoverType datasets, respectively.
\begin{table}
\centering
\begin{tabular}{lccc}
\toprule
Method & OCA ($\uparrow$) & BACC ($\uparrow$) & AVRBACC ($\uparrow$) \\
\midrule
HT & 85.1 & 86.7 & 85.4 \\
ARF & 76.8 & 76.2 & 76.2 \\
HBP & 87.0 & \underline{91.0} & 86.6 \\
ADL & 77.3 & 80.7 & 77.2 \\
CAND & 88.1 & 89.8 & 87.9 \\
EODL & \underline{89.5} & 90.5 & \underline{89.2} \\
\midrule
\textbf{Ours (Ada.)}& \textbf{92.6} & \textbf{92.6} & \textbf{91.7} \\
\bottomrule
\end{tabular}
\caption{Experimental results ($\%$) on Hyperplane dataset. Note: Bold indicates the best result, and underlining denotes the second-best result.}
\label{tab:hp}
\end{table}
\begin{table}
\centering
\begin{tabular}{lccc}
\toprule
Method & OCA ($\uparrow$) & BACC ($\uparrow$) & AVRBACC ($\uparrow$) \\
\midrule
HT & 81.1 & 82.0 & 79.4 \\
ARF & \underline{83.8} & \textbf{86.5} & 81.5 \\
HBP & 82.1 & 82.5 & 78.7 \\
ADL & 60.1 & 49.8 & 49.9 \\
CAND & 79.7 & 82.8 & 76.7 \\
EODL & 82.5 & 83.2 & 79.4 \\
\midrule
\textbf{Ours (Ada.)}& \textbf{85.2} & \underline{83.3} & \textbf{82.6} \\
\bottomrule
\end{tabular}
\caption{Experimental results ($\%$) on SEA dataset. Note: Bold indicates the best result, and underlining denotes the second-best result.}
\label{tab:sea}
\end{table}
\begin{table}
\centering
\begin{tabular}{lr@{/}lr@{/}lc}
\toprule
Method & OCA & MCC & BACC & F1 & AVRBACC \\
\midrule
HT & 77.0&- & 74.3&- & 74.7 \\
ARF & 72.8&- & 76.1&- & 69.3 \\
HBP & 74.8&- & 79.6&- & 70.3 \\
ADL & 74.9&- & 81.6&- & 70.7 \\
CAND & 78.5&- & \underline{83.7}&- & \underline{76.4} \\
ATNN & \underline{83.6}&\underline{66.0} & -&\underline{83.0} & - \\
EODL & 78.0&- & 83.5&- & 76.2 \\
\midrule
\textbf{Ours (Ada.)}& \textbf{86.8}&\textbf{72.9} & \textbf{86.2}&\textbf{86.4} & \textbf{86.9} \\
\bottomrule
\end{tabular}
\caption{Experimental results ($\%$) on Electricity dataset. Note: Bold indicates the best result, and underlining denotes the second-best result.}
\label{tab:elec}
\end{table}
\begin{table}
\centering
\begin{tabular}{lr@{/}lr@{/}lc}
\toprule
Method & OCA & MCC & BACC & F1 & AVRBACC \\
\midrule
HT & 80.1 & - & 73.0 & - & 70.4 \\
ARF & 83.7 & - & 80.9 & - & 64.2 \\
HBP & 91.2 & - & 82.5 & - & 74.7 \\
ADL & 90.5 & - & 86.7 & - & 76.8 \\
CAND & 92.8 & - & 85.4 & - & 79.2 \\
ATNN & 92.7 & \underline{89.0} & - & \underline{87.0} & - \\
EODL & \underline{93.5} & - & \textbf{90.0} & - & \underline{81.0} \\
\midrule
\textbf{Ours (Ada.)}& \textbf{94.3} & \textbf{90.9} & \underline{89.7} & \textbf{90.0} & \textbf{88.2} \\
\bottomrule
\end{tabular}
\caption{Experimental results (\%) on CoverType dataset. Note: Bold indicates the best result, and underlining denotes the second-best result.}
\label{tab:coty}
\end{table}

To make the experimental conclusions more reliable, the results of all baseline methods are directly cited from the corresponding papers and that of Online-BLS with forgetting factor (denoted with Ada.) are averaged over 10 independent experiments. Firstly, ours (Ada.) surpasses typical online machine learning algorithms with a large gap, which benefits from the random feature mapping of BLS that can extract effective broad features. Additionally, ours (Ada.) outperforms existing online deep learning methods in almost all metrics.  
We conjecture that this is due to the fact that the derivation of our method is based on a closed-form solution, which provides a good guarantee of the optimality of the solution updated online. Figure \ref{fig:cd} shows the convergence curves of Online-BLS with and without forgetting factors. 
\begin{figure}[tb]
\centering
\subfloat[Hyperplane]{
\includegraphics[width=1.5in]{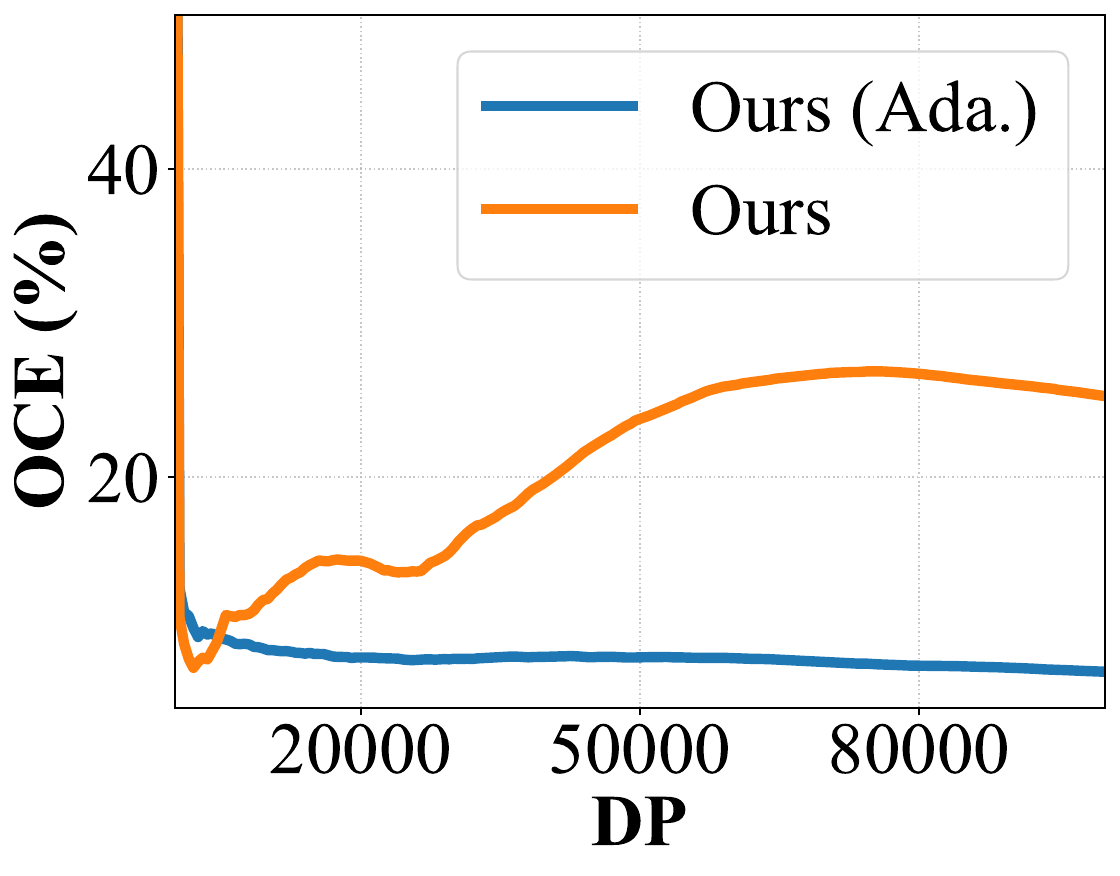}
\label{cd:sf1}
}
\subfloat[SEA]{
\includegraphics[width=1.5in]{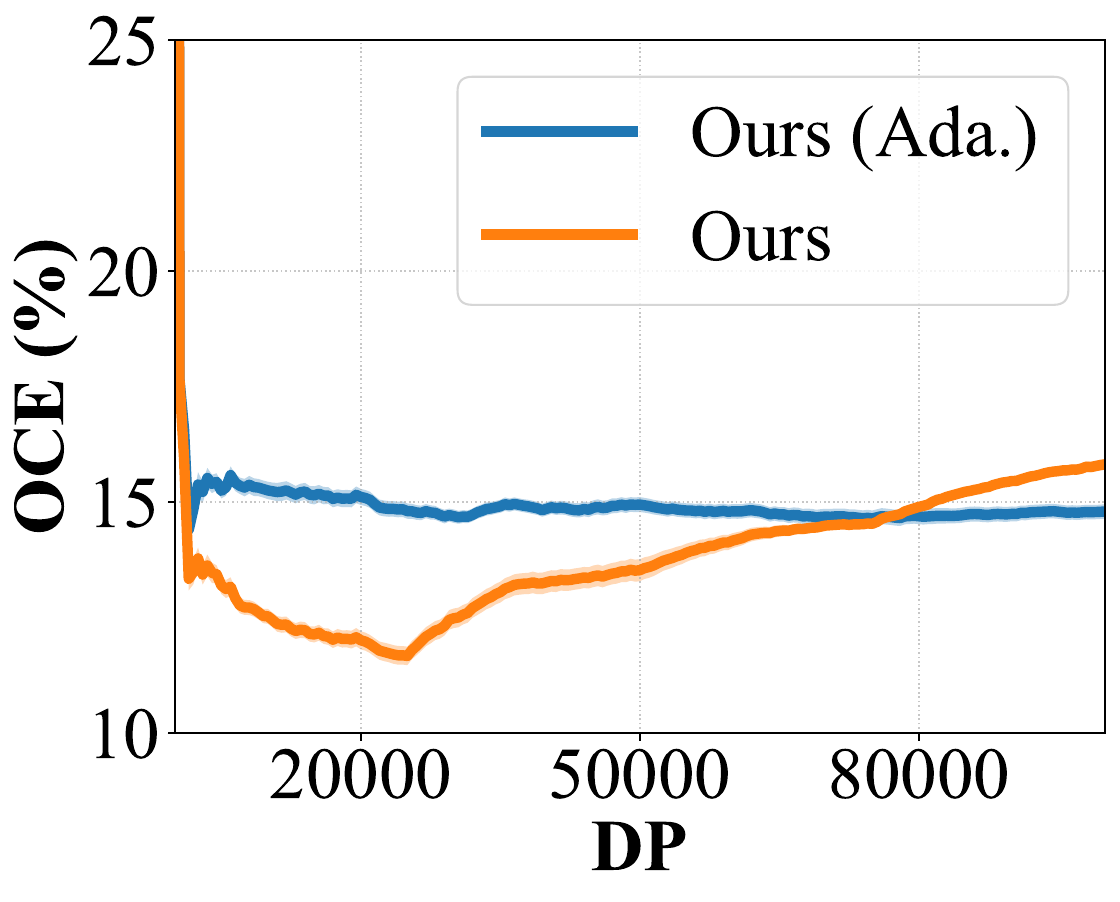}
\label{cd:sf2}
}
\newline
\subfloat[Electricity]{
\includegraphics[width=1.5in]{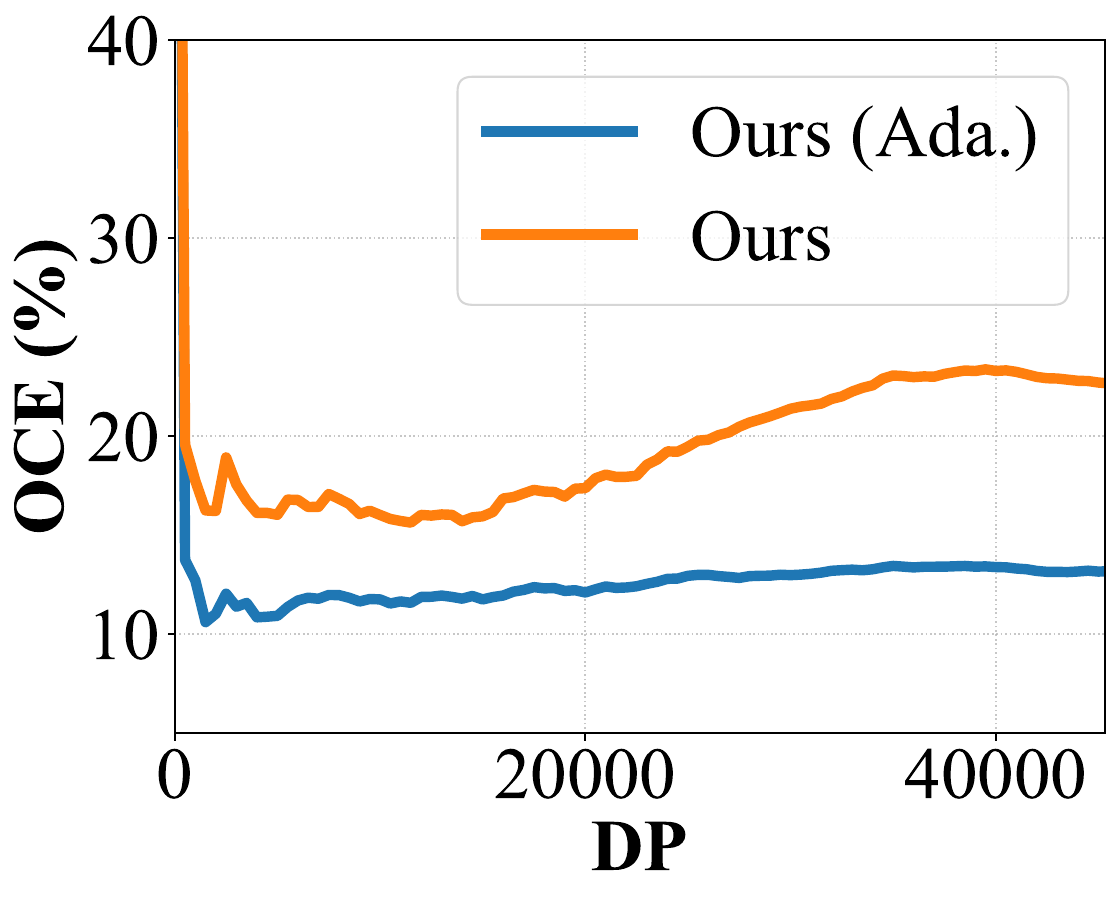}
\label{cd:sf3}
}
\subfloat[CoverType]{
\includegraphics[width=1.5in]{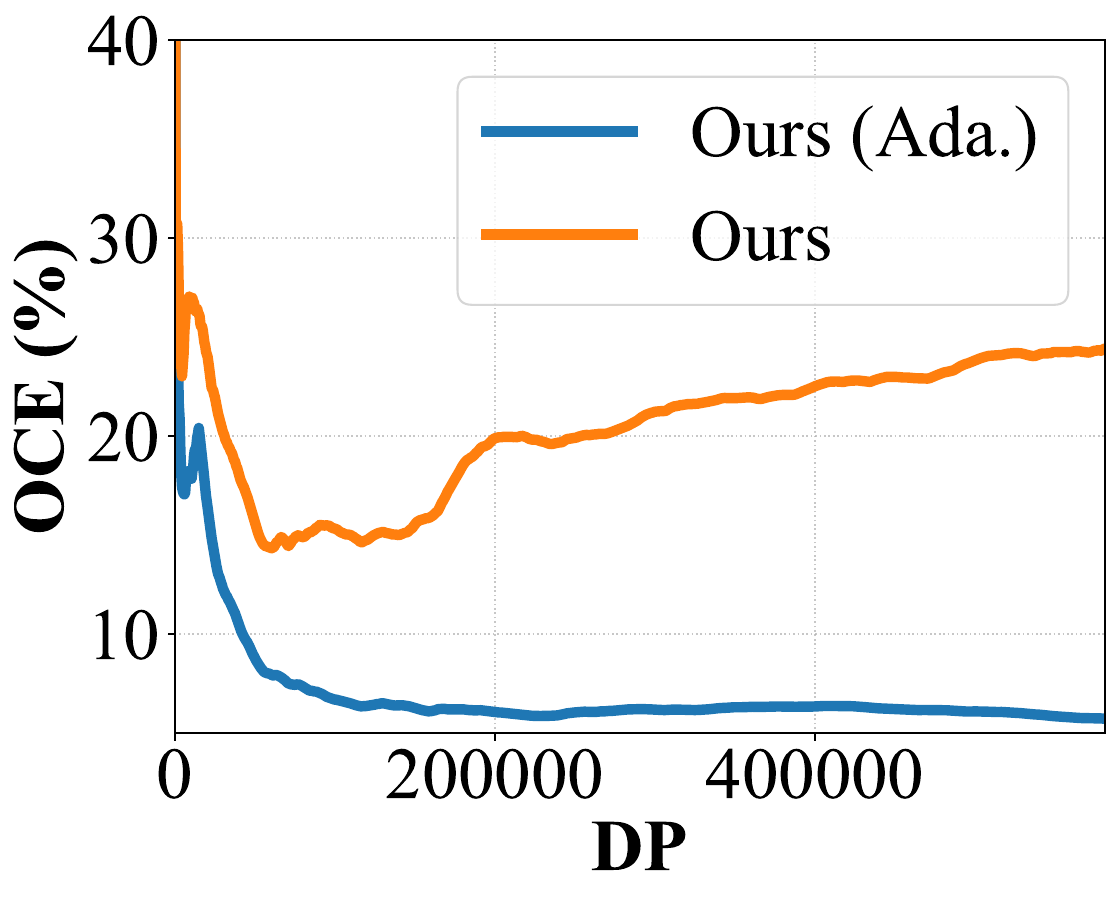}
\label{cd:sf4}
}
\caption{Average and SD of convergence curves of Online-BLS with and without forgetting factor on non-stationary datasets.}
\label{fig:cd}
\end{figure}
From Figure \ref{fig:cd} we find that the average OCE of Online-BLS increases significantly as concept drift occurs, while ours (Ada.) can maintain a low error rate. Thus, we conclude that by adding a forgetting factor to the historical data, our framework can handle the non-stationary data stream classification task well.

\section{Conclusions and Future Works}
This paper presented an accurate and efficient Online-BLS framework to solve the suboptimal model weights problem. On the one hand, an EWEA was proposed to improve model accuracy. By replacing the matrix inverse with Cholesky factorization and forward-backward substitution, our proposed Online-BLS becomes accurate. On the other hand, an EOUS was developed and integrated into EWEA to improve the efficiency of our online update method. Through the integration with EWEA and EOUS, Our Online-BLS framework can be accurate and efficient. Theoretical analysis also proves that Online-BLS has excellent error bound and time complexity. In addition, this paper proposed a simple solution for the concept drift problem by adding a forgetting factor to historical data, which also indicates the flexibility and extensibility of our Online-BLS framework. Experiments on stationary and non-stationary datasets demonstrate the superiority and high efficiency of Online-BLS. This work is the first significant attempt to learn BLS models online. In the future, we will consider the issue of class imbalance in online data streams to improve the model performance further.

% \section*{Ethical Statement}
% There are no ethical issues.

% \section*{Acknowledgments}
% This is acknowledgments
\bibliographystyle{named}
\bibliography{ijcai25}

\clearpage

\appendix
\section{Incremental Broad Learning Systems}
In this section, we introduce the vanilla BLS and several existing incremental broad learning algorithms that can be used for online learning tasks. 

\subsection{BLS and I-BLS}
BLS~\cite{7987745} is essentially equivalent to a ridge regression model, except that it utilizes stochastic transformations and nonlinear mapping to extract nonlinear representations. Assume that the first input sample is $\mathbf{x}_1 \in \mathbb{R}^{d}$, where $d$ is the number of feature dimensions. The $i$-th feature node group $\mathbf{z}_{i} \in \mathbb{R}^{n_1}$ is defined as:
\begin{equation}
\label{eq:z_i_2}
\mathbf{z}_{i}^{\top} = \phi(\mathbf{x}_{1}^{\top}\mathbf{W}_{f_i}+\mathbf{\beta}_{f_i}^{\top}), i=1,2,\cdots,n_{2},
\end{equation}
where $\mathbf{W}_{f_i} \in \mathbb{R}^{d \times n_1}$ and $\mathbf{\beta}_{f_i} \in \mathbb{R}^{n_1}$ are randomly generated. $\phi(\cdot)$ represents the selected activation function. Combining all the feature node groups as $\mathbf{z}^{\top} \triangleq [\mathbf{z}_{1}^{\top}, \mathbf{z}_{2}^{\top}, \cdots, \mathbf{z}_{n_2}^{\top}]$, the $j$-th enhancement node group $\mathbf{h}_{j} \in \mathbb{R}^{n_3}$ can be represented as:
\begin{equation}
\label{eq:h_j_2}
\mathbf{h}_{j}^{\top}  = \sigma(\mathbf{z}^{\top}\mathbf{W}_{e_j}+\mathbf{\beta}_{e_j}^{\top}), j=1,2,\cdots,n_{4},
\end{equation}
where $\mathbf{W}_{e_j} \in \mathbf{R}^{n_1n_2 \times n_3}$ and $\mathbf{\beta}_{e_j} \in \mathbb{R}^{n_3}$ are randomly sampled from a given distribution. $\sigma(\cdot)$ denotes a non-linear activation function. Defining broad feature of $\mathbf{x}_1$ as $\mathbf{a}_{1}^{\top} \triangleq [\mathbf{z}^{\top}|\mathbf{h}_{1}^{\top},\cdots,\mathbf{h}_{n_4}^{\top}]$, the final model weights $\mathbf{W}^{(1)}$ are estimated as:
$\mathbf{W}^{(1)} = (\mathbf{a}_{1}^{\top})^{+}\mathbf{y}_{1}^{\top}$,
where $\mathbf{y}_{1}$ denotes the target vector of $\mathbf{x}_1$ and $(\mathbf{a}_1^{\top})^{+}=(\mathbf{a}_{1}\mathbf{a}_{1}^{\top} + \lambda\mathbf{I})^{-1}\mathbf{a}_{1}$.
$\mathbf{I}$ is an identity matrix, and $\lambda$
is used to control the complexity of BLS. When the $k$-th training sample $\mathbf{x}_k$ arrives, its broad feature $\mathbf{a}_k \in \mathbb{R}^m(m \triangleq n_1n_2+n_3n_4)$ is first obtained via substituting $\mathbf{x}_k$ into equations \ref{eq:z_i} and \ref{eq:h_j}. 
Then, the updated weight matrix $\mathbf{W}^{(k)}$ is:
\begin{equation}
\mathbf{W}^{(k)}=\big([\mathbf{a}_1,\mathbf{a}_2,\cdots, \mathbf{a}_k]^{\top}\big)^{+}[\mathbf{y}_1,\mathbf{y}_2,\cdots,\mathbf{y}_k]^{\top}, \notag
\end{equation}
where
$\big([\mathbf{a}_1,\mathbf{a}_2,\cdots, \mathbf{a}_k]^{\top}\big)^{+}=
\big [
\big([\mathbf{a}_1,\mathbf{a}_2,\cdots, \mathbf{a}_{k-1}]^{\top}\big)^{+}-\mathbf{F}\mathbf{D}^{\top} | \mathbf{F}
\big ]$, $\mathbf{D}^{\top}=\mathbf{a}_{k}^{\top}\big([\mathbf{a}_1,\mathbf{a}_2,\cdots, \mathbf{a}_{k-1}]^{\top}\big)^{+}$, $\mathbf{Q}=\mathbf{a}_{k}^{\top}-\mathbf{D}^{\top}[\mathbf{a}_1,\mathbf{a}_2,\cdots, \mathbf{a}_{k-1}]^{\top}$, and
\begin{equation}
\mathbf{F}=
\begin{cases}
\mathbf{Q}^{+} &{\text{if}\ \mathbf{Q} \ne \mathbf{0}} \\
\big([\mathbf{a}_1,\mathbf{a}_2,\cdots, \mathbf{a}_{k-1}]^{\top}\big)^{+}\mathbf{D}(1+\mathbf{D}^{\top}\mathbf{D})^{-1} &{\text{if}\ \mathbf{Q}=\mathbf{0}}. \notag \\
\end{cases}
\end{equation}

\subsection{RI-BLS}
To remedy the drawbacks of long training time and poor accuracy of I-BLS, RI-BLS~\cite{10533441} has been proposed, which devises an update strategy for the weight matrix by recursively computing two memory matrices. During the initialization training phase (i.e., RI-BLS receives the first training sample), the estimated weight is
\begin{equation}
\mathbf{W}^{(1)}=\big(\mathbf{U}^{(1)} + \lambda \mathbf{I}\big)^{-1}\mathbf{V}^{(1)}, \notag
\end{equation}
where $\mathbf{U}^{(1)}=\mathbf{a}_{1}\mathbf{a}_{1}^{\top}$ and $\mathbf{V}^{(1)}=\mathbf{a}_{1}\mathbf{y}_{1}^{\top}$. 

When the $k$-th sample arrives, we have
\begin{equation}
\mathbf{W}^{(k)}=\big(\mathbf{U}^{(k)} + \lambda \mathbf{I}\big)^{-1}\mathbf{V}^{(k)}, \notag
\end{equation}
where $\mathbf{U}^{(k)} = \mathbf{U}^{(k-1)} + \mathbf{a}_{k}\mathbf{a}_{k}^{\top}$ and $\mathbf{V}^{(k)} = \mathbf{V}^{(k-1)} + \mathbf{a}_{k}\mathbf{y}_{k}^{\top}$.
% \subsection{BLS-CIL}

\subsection{BLS-CIL}
Compared to RI-BLS, BLS-CIL~\cite{10086560} has been proposed, which has an additional class-correlation loss function and was initially designed to address class-incremental tasks. By adding a single training sample at a time, BLS-CIL can be well adapted to online classification tasks. For the initialization training phase, the weight matrix $\mathbf{W}^{(1)}$ can be expressed as
$\mathbf{W}^{(1)} = \big(\mathbf{K}^{(1)}\big)^{-1}\mathbf{a}_1\mathbf{y}_{1}^{\top}$,
where $\mathbf{K}^{(1)}=\mathbf{I}+\mathbf{a}_1\mathbf{a}_1^{\top}$. After BLS-CIL receives a new sample $\mathbf{x}_k$, its weight matrix $\mathbf{W}^{(k)}$ can be formalized as
\begin{equation}
\mathbf{W}^{(k)}=\mathbf{W}^{(k-1)}-\big(\mathbf{K}^{(k)}\big)^{-1}\mathbf{H}\mathbf{W}^{(k-1)}+\big(\mathbf{K}^{(k)}\big)^{-1}\mathbf{a}_k\mathbf{y}_k^{\top}, \notag
\end{equation}
where $\mathbf{K}^{(k)}=\mathbf{K}^{(k-1)}+\mathbf{H}$, $\mathbf{H}=\mathbf{a}_{k}\mathbf{a}_{k}^{\top}+\lambda(\mathbf{a}_{k-1}\mathbf{a}_{k-1}^{\top}-\mathbf{a}_{k}\mathbf{a}_{k-1}^{\top}+\mathbf{a}_{k}\mathbf{a}_{k}^{\top}-\mathbf{a}_{k-1}\mathbf{a}_k^{\top})$, and
$\lambda=
\begin{cases}
-\lambda_{1} &{\text{if}\ \mathbf{y}_k \ne \mathbf{y}_{k-1}}\\
+\lambda_{2} &{\text{if}\ \mathbf{y}_k=\mathbf{y}_{k-1}}, \notag\\
\end{cases}$
where $\lambda_{1}$ and $\lambda_{2}$ are positive hyper-parameters which are responsible for intra-class compact term and inter-class sparse term losses, respectively.

\section{Datasets}
A total of $10$ streaming datasets are used in this paper. Among them, there are $6$ stationary datasets and $4$ non-stationary datasets with concept drift. In the following, we introduce them in detail.

\subsection{Stationary Datasets}
\noindent\textbf{Image Segment\footnote{\url{https://archive.ics.uci.edu/dataset/50/image+segmentation}}}: The Image Segment dataset is an image classification dataset described by $19$ high-level numerical attributes. Its samples are randomly sampled from a dataset containing $7$ outdoor images. The images are manually segmented and assigned a class label to each pixel. Each instance contains a $3 \times 3$ region. It has $2,310$ samples containing seven classes $\textit{brickface}$, $\textit{sky}$, $\textit{foliage}$, $\textit{cement}$, $\textit{window}$, $\textit{path}$, and $\textit{grass}$.

\noindent\textbf{USPS\footnote{\url{https://datasets.activeloop.ai/docs/ml/datasets/usps-dataset/}}}: The USPS dataset is used for the task of handwritten digit classification. The USPS contains $9,298$ samples from $10$ classes ranging from $0$ to $9$. Each sample is a $16 \times 16$ grayscale image. To facilitate BLS modeling, we flatten each sample into a one-dimensional vector.

\noindent\textbf{Letter\footnote{\url{https://archive.ics.uci.edu/dataset/59/letter+recognition}}}: The Letter dataset is a capital letter in the English alphabet classification dataset. The Letter dataset contains $20,000$ samples. Each sample contains $16$ primary numeric attributes (statistical moments and edge counts), which are then scaled to an integer value range of $0$ to $15$.

\noindent\textbf{Adult\footnote{\url{https://archive.ics.uci.edu/dataset/2/adult}}}: The Adult dataset is used for a binary classification task, which determines whether a person earns more than $50k$ a year based on $14$ attributes. It has both numeric and categorical attributes. We first convert categorical attributes to integers that the model can handle easily. Then, samples with missing values are removed and the training and test sets are merged. Eventually, the Adult dataset has a total of $45,222$ samples.

\noindent\textbf{Shuttle\footnote{\url{https://archive.ics.uci.edu/dataset/148/statlog+shuttle}}}: The Shuttle dataset is a seven-class dataset, including $\textit{Rad Flow}$, $\textit{Fpv Close}$, $\textit{Fpv Open}$, $\textit{High}$, $\textit{Bypass}$, $\textit{Bpv Close}$, and $\textit{Bpv Open}$. There are $9$ numerical attributes. Among them, the first attribute is time.
Thus, we remove the first column of attributes and use the remaining 8 attributes to predict one of the seven categories.

\noindent\textbf{MNIST\footnote{\url{http://yann.lecun.com/exdb/mnist/}}}: The MNIST dataset is a handwritten digit recognition benchmark. Each sample is a grayscale image with $28 \times 28$ pixels. The training and test sets are merged to yield a dataset containing $70,000$ samples uniformly distributed from 0 to 9. To adapt to our model, each sample is flattened into a vector and each attribute is normalized to be between $0$ and $1$.

\subsection{Non-stationary Datasets}
For non-stationary datasets, we choose two synthetic datasets and two real datasets. These two synthetic datasets, Hyperplane and SEA, are generated using the River\footnote{\url{https://github.com/online-ml/river}} package. The remaining two real datasets are originated from the real world, often with complex unpredictable concept drift.

\noindent\textbf{Hyperplane}: The Hyperplane dataset is an artificial binary classification task. Its task is to divide points in a $d$-dimensional space into two parts by a $d-1$ dimensional hyperplane. A sample is labelled positive if it satisfies $\sum^{d}_{i=1} w_i x_i > w_0$ and negative if the opposite is true. By smoothly changing the parameters of the classification hyperplane, concept drift can be added to the generated data stream. Specifically, following~\cite{su2024elastic}, we generate a data stream containing $20$ features. Its noise level and drift magnitude are set to $1\%$ and $0.5\%$, respectively.

\noindent\textbf{SEA}: The SEA dataset is also an artificial binary classification dataset. Each sample contains three features, of which only the first two are relevant. If the sum of the first two features exceeds a certain threshold, then it is a positive example, otherwise it is a negative example. There are $4$ thresholds to choose from. Concept drift is introduced by switching thresholds at the $25,000$-th, $50,000$-th, and $75,000$-th data points. Moreover, following~\cite{su2024elastic}, the noise level we introduce in the process of data stream generation is $10\%$.

\noindent\textbf{Electricity\footnote{\url{https://sourceforge.net/projects/moa-datastream/files/Datasets/Classification/elecNormNew.arff.zip/download}}}: The Electricity dataset is used as a binary classification task, where the goal is to predict if the electricity price of Australian New South Wales
will go up or down. In this market, prices are not fixed and are affected by demand and supply of the market. They are set every
five minutes. Each sample originally had $8$ attributes. Following~\cite{su2024elastic}, we delete the attributes of date and time located in the first two columns and keep the remaining $6$ attributes. This dataset contains $45,312$ instances and its concepts fluctuate over time.

\noindent\textbf{CoverType\footnote{\url{https://archive.ics.uci.edu/dataset/31/covertype}}}: The CoverType dataset predicts $7$ forest cover types from $54$ feature value of each instance. The feature types include integer and category types. First, we convert the categorical type to an integer for ease of handling. Then, the CoverType dataset has $581,012$ instances with unknown concept drift.

\section{Metrics}
In this section, we introduce the six metrics used in this paper.

\subsection{Online Cumulative Accuracy}
The online cumulative accuracy is used to evaluate the overall performance of online learning models and is defined as
\begin{equation}
OCA=\frac{1}{n}\sum_{k=1}^{n}\mathbb{I}_{(\mathbf{\hat{y}}_k=\mathbf{y}_k)}, \notag
\end{equation}
where $\mathbb{I}$ is the indicator function. $\mathbf{\hat{y}}_{k}$ and $\mathbf{y}_k$ represent the prediction and ground truth of the $k$-th instance, respectively.

\subsection{Online Cumulative Error}
The online cumulative error  is used to evaluate the overall error rate during the online learning process and is defined as follows:
\begin{equation}
OCE=\frac{1}{n}\sum_{k=1}^{n}\mathbb{I}_{(\mathbf{\hat{y}}_k\ne\mathbf{y}_k)}, \notag
\end{equation}
where $\mathbb{I}$ is the indicator function. $\mathbf{\hat{y}}_{k}$ and $\mathbf{y}_k$ represent the prediction and ground truth of the $k$-th instance, respectively. It should be noted that the sum of the online cumulative accuracy and the online cumulative error at the same time step is $1$.

\subsection{Balanced Accuracy}
The balanced accuracy treats all categories equally and thus is more reliable in data streams with imbalanced classes. It is defined as follows:
\begin{equation}
BACC_k=\left (\sum_{i=1}^{c}\frac{n_{c_i}^{T}}{n_{c_i}}\right)\big/c, \notag
\end{equation}
where $n_{c_i}^{T}$ and $n_{c_i}$ denote the number of correctly predicted instances and the total number of instances belonging to the $i$-th class, respectively. 
$c$ represents the total number of classes, and $k$ is the current time step.

\subsection{Average Balanced Accuracy}
The average balanced accuracy is the average of the balanced accuracy at each time step, which can reflect the overall performance of the online learning model throughout the learning process. It has the following definition:
\begin{equation}
AVRBACC=\frac{\sum_{k=1}^{n}BACC_k}{n}, \notag
\end{equation}
where $n$ is the number of instances received by online learning model.
\subsection{F1 Score}
Since most of the datasets we use are class-balanced, F1 in this context refers to Macro F1 Score. Its definition is  
\begin{equation}
F1=\frac{1}{c}\sum_{i=1}^{c}\left(2\frac{R_i \times P_i}{R_i + P_i}\right), \notag
\end{equation}
where $R_i$ and $P_i$ denote the recall and precision of class $i$, respectively.

\subsection{Matthews Correlation Coefficient}
The matthews correlation coefficient is generally considered a better metric than F1 score and accuracy. Its essence is a correlation coefficient value between $-1$ and $1$. The coefficient $+1$, $0$, $-1$ represent a perfect prediction, an average random prediction, and an inverse prediction, respectively. Taking binary classification as an example, its formula is
\begin{align}
MCC&= \notag\\
&\frac{TP \times TN - FP \times FN}{\sqrt{(TP + FP)(TP + FN)(TN + FP)(TN + FN)}}, \notag
\end{align}
where $TP$, $TN$, $FP$, and $FN$ denote the number of true positive, the number of ture negative, the number of false positive, and the number of false negative, respectively. In fact, we call sklearn's function\footnote{\url{https://scikit-learn.org/stable/modules/generated/sklearn.metrics.matthews_corrcoef.html}} to acquire the matthews correlation coefficient.

\end{document}